\documentclass[final,12pt]{colt2024} 

\title[$(\epsilon, u)$-Adaptive Regret Minimization in Heavy-Tailed Bandits]{$(\epsilon, u)$-Adaptive Regret Minimization in Heavy-Tailed Bandits}
\usepackage{times}

\usepackage[utf8]{inputenc} 
\usepackage[T1]{fontenc}    
\usepackage{booktabs}       
\usepackage{amsfonts}       
\usepackage{nicefrac}       
\usepackage{microtype}      
\usepackage{xcolor}         

\usepackage{microtype}
\usepackage{graphicx}
\usepackage{booktabs}
\usepackage{hyperref}
\usepackage{custom_algorithmnames}
\usepackage{geometry}
\let\proof\relax
\let\endproof\relax
\usepackage{amscd}
\usepackage{bbm}
\usepackage{mathtools}
\usepackage{bm}
\usepackage{color}
\usepackage{comment}
\usepackage{thmtools}
\usepackage{thm-restate}
\usepackage{amsmath}
\usepackage{caption}
\usepackage{subcaption}
\usepackage{xspace}
\usepackage{stmaryrd}
\usepackage{dsfont}
\usepackage{enumitem}
\usepackage{pifont}
\usepackage{tabulary}
\usepackage{multirow}
\usepackage{colortbl}
\usepackage{arydshln}

\usepackage{lscape}

\newcommand{\cmark}{\textcolor{green!50!black}{\ding{51}}}%
\newcommand{\xmark}{\textcolor{red!70!black}{\ding{55}}}%

\usepackage{tablefootnote}
\newcounter{daggerfootnote}

\usepackage[capitalize,noabbrev]{cleveref}

\newcommand{\As}{\mathcal{A}}

\newcommand{\E}{\mathop{\mathbb{E}}}

\allowdisplaybreaks[4]

\usepackage{tikz}
\usepackage{tikzscale}
\usepackage{pgfplots}
\usetikzlibrary{matrix}
\usepgfplotslibrary{groupplots}
\pgfplotsset{compat=newest}

\DeclareRobustCommand{\ie}{\emph{i.e.},\@\xspace}

\newcommand{\vnu}{\bm{\nu}}

\definecolor{citrine}{rgb}{1, 0.95, 0.5}

\newcommand{\dsb}[1]{\llbracket #1 \rrbracket}

\DeclareMathOperator*{\argmax}{arg\,max}

\makeatletter
\newcommand{\namelabel}[2]{%
  \phantomsection
  \renewcommand{\@currentlabel}{#1}
  \label{#2}
}

\coltauthor{%
 \Name{Gianmarco Genalti} \Email{gianmarco.genalti@polimi.it}\\
 \Name{Lupo Marsigli} \Email{lupo.marsigli@mail.polimi.it}\\
 \Name{Nicola Gatti} \Email{nicola.gatti@polimi.it}\\
 \Name{Alberto Maria Metelli} \Email{albertomaria.metelli@polimi.it}\\
 \addr Department of Computer Science, Politecnico di Milano
}

\newenvironment{proof}{\par\noindent{\bf Proof\ }}{\hfill$\blacksquare$\\[2mm]}

\begin{document}

\setlength{\abovedisplayskip}{4pt}
\setlength{\belowdisplayskip}{4pt}
\setlength{\textfloatsep}{8pt}

\maketitle

\begin{abstract}
  Heavy-tailed distributions naturally arise in several settings, from finance to telecommunications. 
  While regret minimization under subgaussian or bounded rewards has been widely studied, learning with heavy-tailed distributions only gained popularity over the last decade. In this paper, we consider the setting in which the reward distributions have finite absolute raw moments of maximum order $1+\epsilon$, uniformly bounded by a constant $u<+\infty$, for some $\epsilon \in (0,1]$. In this setting, we study the regret minimization problem when $\epsilon$ and $u$ are unknown to the learner and it has to adapt. First, we show that adaptation comes at a cost and derive two negative results proving that the same regret guarantees of the non-adaptive case cannot be achieved with no further assumptions. Then, we devise and analyze a fully data-driven trimmed mean estimator and propose a novel adaptive regret minimization algorithm, \algnameshort, that leverages such an estimator. Finally, we show that \algnameshort is the first algorithm that, under a known distributional assumption, enjoys regret guarantees nearly matching those of the non-adaptive heavy-tailed case.
\end{abstract}


\section{Introduction}
In this paper, we investigate the stochastic \emph{multi-armed bandit} problem~\citep[MAB,][]{auer2002finite, lattimore2020bandit} under the assumption of \emph{heavy-tailed} (HT) reward distributions.
In the stochastic MAB setting~\citep{robbins1952some}, a learner has access to a set of $K \in \mathbb{N}_{\ge 2}$ actions (\emph{i.e.},~\emph{arms}). Each arm $i\in[K] \coloneqq \{1,\dots,K\}$ is associated with a reward probability distribution $\nu_i \in \Delta(\mathbb{R})$,\footnote{Given a measurable space $(\mathcal{X},\mathcal{F})$, we denote with $\Delta(\mathcal{X})$ the set of probability measures over $(\mathcal{X},\mathcal{F})$.} having finite expected value $\mu_i \coloneqq \E_{X \sim \nu_i}[X]$ (\emph{i.e.},~\emph{expected reward}). We denote with $\vnu = (\nu_1,\dots,\nu_K) \in \Delta(\mathbb{R})^K$ a specific bandit instance. Let $T \in \mathbb{N}$ be the learning horizon, at every round $t\in[T]$, the learner selects an arm $I_t \in [K]$ and, in response, the environment reveals the $X_t \sim \nu_{I_t}$ (\emph{i.e.}, \emph{reward}) sampled from the distribution $\nu_{I_t}$. The performance of a learner running an algorithm $\texttt{Alg}$ is quantified by the \textit{(expected cumulative) regret} over $T$ rounds, defined as:
\begin{equation}
\label{eq:regret}
    R_T(\texttt{Alg},\vnu) \coloneqq  T \mu_1 - \mathbb{E}\left[\sum_{t=1}^T \mu_{I_t}\right] = \mathbb{E}\left[\sum_{t=1}^T \Delta_{I_t}\right],
\end{equation}
where we assume, without loss of generality, that $1$ is the optimal arm and $\Delta_i \coloneqq \mu_1 - \mu_i$ is the suboptimality gap of arm $i \in [K]$, and the expectation is taken w.r.t. the randomness of the reward and the possible randomness of the algorithm.



Although most of the literature in stochastic MABs usually assumes a convenient tail property of the reward distributions, like \emph{subgaussian}~\citep{lattimore2020bandit} or \emph{bounded}~\citep{auer2002finite} rewards,
in many practical scenarios, such as financial environments ~\citep{gagliolo2011algorithm} or network routing problems ~\citep{liebeherr2012delay}, where uncertainty has a significant impact, {heavy-tailed}  distributions naturally arise. In these cases, the tails decay slower than a Gaussian, the moment-generating function is no longer finite, and the moments of all finite orders might not exist. This prevents the application of standard concentration tools, such as the Hoeffding's inequality~\citep{boucheron2013concentration}, calling for more complex technical tools.

This work investigates the regret minimization problem for MAB with HT reward distributions, according to the setting introduced in the seminal work \citep{bubeck2013bandits}. We assume that the \emph{absolute raw moments} of the reward distributions of order up to $1+\epsilon$, with $\epsilon \in (0,1]$ (\emph{i.e.,} moment order) is finite and uniformly bounded by a constant $u\in\mathbb{R}_{\ge 0}$ (\emph{i.e.,} moment bound), namely:
\begin{align}
    \boldsymbol{\nu} \in \htsetk \coloneqq \left\{\boldsymbol{\nu} \in \Delta(\mathbb{R})^K \,:\,  \E_{X \sim \nu_i}\left[|X|^{1+\epsilon}\right] \le u,~~\forall i \in [K] \right\},
\end{align}
In \cite{bubeck2013bandits}, the authors assume that \emph{$\epsilon$ and $u$ are known to the learner}. They show that if the variance is finite (\emph{i.e.}, $\epsilon = 1$), but the higher order moments are not, the same (apart from constants)  instance-dependent regret guarantees of order $O\big(\sum_{i : \Delta_i > 0}\frac{\sigma^2}{\Delta_i}\log T\big)$
attained in the subgaussian setting ~\citep{lattimore2020bandit} can be achieved. However, for general $\epsilon\in(0,1)$, the instance-dependent regret becomes of order
$
    O \big(\sum_{i: \Delta_i > 0}\big(\frac{u}{\Delta_i}\big)^{1/\epsilon}\log T\big),
$
showing the detrimental effect of $\epsilon$ on the dependence of the suboptimality gaps.
Moreover, they show that these regret guarantees are tight (up to constant terms) by deriving the corresponding asymptotic lower bound. From a worst-case regret perspective, the presented results translate into a regret bound of order
$
    \widetilde{O}\big(K^{\frac{\epsilon}{1+\epsilon}} (uT)^{\frac{1}{1+\epsilon}}\big),
$
for sufficiently large $T$, matching the lower bound, up to logarithmic terms. This regret bound degenerates to linear when $\epsilon \rightarrow 0$, \emph{i.e.,} when only absolute moment of order $1$ exists.
However, these matching results are obtained thanks to the knowledge of both $\epsilon$ and $u$, \emph{i.e.,} by \emph{non-adaptive} algorithms. Indeed, $\epsilon$ and $u$ are needed by the \emph{algorithm} to drive exploration via the optimistic index, and, in some cases, to construct the expected reward \emph{estimator} too.

Nevertheless, recent works have shed light on the possibility of removing this knowledge at the cost of additional assumptions~\citep[e.g.,][]{lee2020optimal,ashutosh2021bandit,huang2022adaptive}. In particular, \cite{huang2022adaptive} introduce
the \textit{truncated non-negativity} assumption designed for losses, that, by converting it for rewards, leads to the following \emph{truncated non-positivity} assumption.
\begin{restatable}[Truncated Non-Positivity]{ass}{Assumption}\label{ass:truncated_np}
Let $\bm{\nu}$ be a bandit. For the optimal arm $1$, we have: 
\begin{equation}
    \label{eq:assumption}
    \mathbb{E}_{X \sim \nu_1}\left[X \mathbbm{1}_{\{|X|>M\}}\right] \le 0, \quad \forall M \ge 0.
\end{equation}
\end{restatable}
This assumption requires that the optimal arm $1$ has a larger probability mass on the negative semi-axis but still allows the distribution to have an arbitrary support covering, potentially, the whole $\mathbb{R}$. 
To the best of the authors' knowledge, this is the only assumption in literature truly independent of the values of $\epsilon$ and $u$. Additionally, as discussed by~\citet{huang2022adaptive}, it is relatively weak if compared to other standard assumptions in the bandit literature. Under Assumption~\ref{ass:truncated_np}, \emph{without the knowledge of $\epsilon$ and $u$}, \citet{huang2022adaptive} provide an $(\epsilon,u)$-\emph{adaptive}\footnote{We use the word \emph{adaptive} to qualify algorithms that do not know the values of $\epsilon$ and/or $u$.} regret minimization algorithm, \texttt{AdaTINF}, that succeeds in matching the worst-case regret lower bound of~\citep{bubeck2013bandits} derived for the non-adaptive case. However, no instance-dependent analysis is provided of \texttt{AdaTINF}\footnote{\citet{huang2022adaptive} actually provide algorithm \opthtinf with an instance-dependent analysis that, however, does not match the asymptotic lower bound of~\citep{bubeck2013bandits}.} and the following research questions remain open:

\begin{enumerate}[label={\textbf{Research Question \arabic*}~}, leftmargin=3.7cm, ref={{Research Question \arabic*}}, topsep=-1pt, noitemsep]
    \item \label{question1} \emph{Is Assumption~\ref{ass:truncated_np} needed to devise $(\epsilon,u)$-adaptive algorithms (with unknown $(\epsilon,u)$) matching the worst-case lower bound of order $\Omega\big(K^{\frac{\epsilon}{1+\epsilon}} (uT)^{\frac{1}{1+\epsilon}}\big)$ (\emph{i.e.,} as if we knew $(\epsilon,u)$)?}
    \item \label{question2} \emph{Is it possible, under Assumption~\ref{ass:truncated_np}, to devise $(\epsilon,u)$-adaptive algorithms (with unknown $(\epsilon,u)$) matching   the instance-dependent regret lower bound of order ~$\Omega \big(\sum_{i: \Delta_i > 0}\big(\frac{u}{\Delta_i}\big)^{1/\epsilon}\log T\big) $ (\emph{i.e.,} as if we knew $(\epsilon,u)$)? }
\end{enumerate}


\noindent\textbf{Original Contributions.}~~In this paper, we investigate the regret minimization problem in heavy-tailed bandits giving up the knowledge of $\epsilon$ and $u$. Specifically, we address \ref{question1} and \ref{question2}. The original contributions of the paper are summarized as follows:
\begin{itemize}[noitemsep, leftmargin=*, topsep=0pt]
    \item In Section~\ref{sec:lower_bounds}, we address \ref{question1}, by characterizing the challenges of the regret-minimization problem in HT bandits without knowing $\epsilon$ and $u$. In particular, we provide two \emph{negative results} (Theorems~\ref{thr:u_adaptive_lb} and~\ref{thr:epsilon_adaptive_lb}), showing that, without any additional assumption, there exists no $(\epsilon,u)$-adaptive algorithm that archives the same worst-case regret guarantees as if $\epsilon$ or $u$ were known~\citet{bubeck2013bandits}. These results provide a first justification of Assumption~\ref{ass:truncated_np}. Furthermore, we show how Assumption~\ref{ass:truncated_np} does not reduce the complexity of the regret minimization problem even in the non-adaptive case (Theorem~\ref{thr:AssLowerBound}). These results rely on accurately defined HT bandit instances and information theory tools for deriving the lower bounds.
    \item In Section~\ref{sec:tmest}, we enhance the \emph{trimmed mean} estimator, commonly used in HT bandits, to make it fully data-driven. Indeed, in the seminal paper~\citep{bubeck2013bandits}, both ($i$) the trimming threshold and ($ii$) the upper confidence bound were computed thanks to the knowledge of $\epsilon$ and $u$. Taking inspiration from Huber regression \citep{wang2021new}, we overcome  ($i$) the need for $\epsilon$ and $u$ in the estimator by developing a novel approach to recover an estimated threshold via \emph{root-finding}. Leveraging an analysis based on the \emph{self-bounding functions}~\citep{maurer2006concentration, maurer2009empirical}, we control the accuracy of the estimated threshold (Lemma~\ref{thr:M_hat_bound}). In particular, we show that our threshold underestimates (in high probability) the one proposed by~\citet{bubeck2013bandits}. Furthermore, we overcome ($ii$) by resorting to \emph{empirical Bernstein inequality} \citep{maurer2009empirical}. This way, differently from~\citep{bubeck2013bandits}, we use the empirical variance to eliminate the dependence on $\epsilon$ and $u$ in the upper confidence bound (Lemma~\ref{thr:concentration_ineq}), preserving the desirable concentration properties of the delivered estimate (Theorem~\ref{thr:conc_ineq_2}).
    \item In Section~\ref{sec:algo}, we address \ref{question2}, by devising and analyzing a novel $(\epsilon,u)$-adaptive regret minimization algorithm, \algname (\algnameshort, Algorithm~\ref{alg:alg}), that operates without the knowledge of $\epsilon$ and $u$. \algnameshort is an \emph{optimistic anytime} algorithm that builds upon \robustucb of~\citet{bubeck2013bandits}, leveraging our trimmed mean estimator with estimated threshold. First, we show that, under Assumption~\ref{ass:truncated_np}, \algnameshort attains an instance-dependent regret bound of order $O\big( \sum_{i :\Delta_i > 0} \big(\big(\frac{u}{\Delta_i}\big)^{1/\epsilon} + \frac{\Delta_i}{\mathbb{P}_{\nu_i}(X\neq 0)}\big)\log T \big)$ (Theorem~\ref{thr:upper_bound}). This result shows that \algnameshort nearly matches the instance-dependent lower bound of~\cite{bubeck2013bandits} for the non-adaptive case, apart from the second logarithmic term, which, however, does not depend on the reciprocal of the suboptimality gaps, and originates from an additional forced exploration needed for computing the empirical threshold.\footnote{A similar additional term appears in the instantiation of \robustucb with Catoni estimator~\citep{bubeck2013bandits}.} Moreover, we show that \algnameshort suffers a worst-case regret bound of order $\widetilde{O}\big( K^{\frac{\epsilon}{1+\epsilon}} (uT)^{\frac{1}{1+\epsilon}} \big)$ (Theorem~\ref{thr:inst_indep_regret}), matching, up to logarithmic terms, the minimax lower bound of the non-adaptive case~\citep{bubeck2013bandits}. To the best of authors' knowledge, \algnameshort is the first $(\epsilon,u)$-adaptive algorithm for HT bandits that nearly matches both the instance-dependent and worst-case lower bounds of the non-adaptive case, under conditions (Assumption~\ref{ass:truncated_np}) not explicitly formulated in terms of $\epsilon$ and $u$.
\end{itemize}

\noindent In Section \ref{sec:related_works} we provide an up-to-date literature review on \textit{adaptivity} in heavy-tailed bandits. The proofs of the results presented in the main paper are reported in Appendix~\ref{apx:proofs}.

\section{Related Works}
\label{sec:related_works}
During the last ten years, the stochastic heavy-tailed bandit problem has been steadily increasing in popularity. In this section, we summarize the main contributions, with a particular focus on partially adaptive approaches. Table~\ref{tab:comparison} provides a comprehensive comparison.

\newcommand{\thickhline}{
    \noalign {\ifnum 0=`}\fi \hrule height 1pt
    \futurelet \reserved@a \@xhline
}
\begin{table}[!h]
\label{tab:comparison}
 \centering
 \renewcommand{\arraystretch}{1.2}
  \setlength{\tabcolsep}{2pt}
  \resizebox{.95\textwidth}{!}{
  \thinmuskip=1mu
\medmuskip=1mu
\thickmuskip=1mu
  \small
  \smallskip
  \begin{tabular}{|l|c:c|c:c|c:c|c:c|c|}
  \thickhline
   \textbf{Algorithm} & \multicolumn{4}{c|}{\textbf{Regret Bounds}} & \multicolumn{2}{c|}{$\boldsymbol{\epsilon}$\textbf{-adaptive}} & \multicolumn{2}{c|}{$\boldsymbol{u}$\textbf{-adaptive} } & \textbf{Assumption} \\
   \cline{2-9}
   & Instance-dependent & \rotatebox{90}{
   \hspace{-.2cm}$\begin{array}{l}
   \vspace{-.5cm}\\\text{Matching?{$^\mathsection$}}\\\vspace{-.5cm}
   \end{array}$
   } & Worst-case &  \rotatebox{90}{
   \hspace{-.2cm}$\begin{array}{l}
   \vspace{-.5cm}\\\text{Matching?{$^\mathparagraph$}}\\\vspace{-.5cm}
   \end{array}$
   } & \rotatebox{90}{
   \hspace{-.2cm}$\begin{array}{l}
   \vspace{-.4cm}\\\text{Estimator}\\\vspace{-.4cm}
   \end{array}$
   } & \rotatebox{90}{
   \hspace{-.2cm}$\begin{array}{l}
   \vspace{-.4cm}\\\text{Algorithm}\\\vspace{-.4cm}
   \end{array}$
   } & \rotatebox{90}{
   \hspace{-.2cm}$\begin{array}{l}
   \vspace{-.4cm}\\\text{Estimator}\\\vspace{-.4cm}
   \end{array}$
   } & \rotatebox{90}{
   \hspace{-.2cm}$\begin{array}{l}
   \vspace{-.4cm}\\\text{Algorithm}\\\vspace{-.4cm}
   \end{array}$
   } & \\
   \hline
    $\begin{array}{l}
         \text{\robustucb~- TM}  \vspace{-.2cm}\\
         \text{\citep{bubeck2013bandits}} 
    \end{array}$
   & $\displaystyle\sum_{i: \Delta_i >0} \left(\frac{u}{\Delta_i}\right)^{1/\epsilon}\log T$ & \cmark 
   & $\displaystyle K^\frac{\epsilon}{1+\epsilon}u^\frac{1}{1+\epsilon}T^\frac{1}{1+\epsilon}(\log T)^\frac{\epsilon}{1+\epsilon}$ & \cmark
   & \xmark
   & \xmark
   & \xmark
   & \xmark
   & --- \\
   \hline
   $\begin{array}{l}
         \text{\robustucb~- MoM $^\star$}  \vspace{-.2cm}\\
         \text{\citep{bubeck2013bandits}} 
    \end{array}$
   & $\displaystyle\sum_{i: \Delta_i >0} \left(\frac{v}{\Delta_i}\right)^{1/\epsilon}\log T$ & \cmark 
   & $\displaystyle K^\frac{\epsilon}{1+\epsilon}v^\frac{1}{1+\epsilon}T^\frac{1}{1+\epsilon}(\log T)^\frac{\epsilon}{1+\epsilon}$ & \cmark 
   & \cmark
   & \xmark
   & \cmark
   & \xmark
   & --- \\
   \hline
   $\begin{array}{l}
         \text{\robustucb~- Catoni $^\star$}  \vspace{-.2cm}\\
         \text{\citep{bubeck2013bandits}} 
    \end{array}$
   & $\displaystyle\sum_{i: \Delta_i >0} \left(\frac{v}{\Delta_i}+\Delta_i\right)\log T$ & \cmark 
   & $\displaystyle\sqrt{vKT\log T}+ \sum_{i:\Delta_i>0} \Delta_i\log T$ & \cmark 
   & \xmark
   & \xmark 
   & \cmark
   & \xmark 
   & $\epsilon = 1$ only \\
   \hline
   $\begin{array}{l}
         \text{\rmoss}  \vspace{-.2cm}\\
         \text{\citep{wei2020minimax}} 
    \end{array}$
   & $\displaystyle\sum_{i: \Delta_i>0} \log \Big(\frac{T\Delta_i^\frac{1+\epsilon}{\epsilon}}{K}\Big)\frac{1}{\Delta_i^{1/\epsilon}}$ & \cmark 
   & $\displaystyle K^\frac{\epsilon}{1+\epsilon}u^\frac{1}{1+\epsilon}T^\frac{1}{1+\epsilon}$ & \cmark 
   & \xmark
   & \xmark
   & \xmark
   & \xmark
   & ---\\
   \hline
    $\begin{array}{l}
         \text{\klinfucb $^\dagger$}  \vspace{-.2cm}\\
         \text{\citep{agrawal2021regret}} 
    \end{array}$
   & $\displaystyle\sum_{i : \Delta_i>0} \frac{\log T}{D_{\text{KL}}^{\text{inf}}(\nu_i, \mu_1)}$ & \cmark 
   & --- & ---
   & \cmark
   & \xmark
   & \cmark
   & \xmark
   & --- \\
   \hline
   $\begin{array}{l}
         \text{\apetwo}  \vspace{-.2cm}\\
         \text{\citep{lee2020optimal}} 
    \end{array}$
   & $\displaystyle\sum_{i : \Delta_i >0} \left(e^u + (T\Delta_i^\frac{1+\epsilon}{\epsilon} \log K  )^\frac{1+\epsilon}{\epsilon \log K }\right)\frac{1}{\Delta_i^{1/\epsilon}} $ & \xmark
   & $\displaystyle K^\frac{\epsilon}{1+\epsilon}u^\frac{1}{1+\epsilon}T^\frac{1}{1+\epsilon}\log T e^u$ & \xmark
   & \xmark
   & \xmark
   & \cmark
   & \cmark
   & --- \\
   \hline
    $\begin{array}{l}
         \text{\mrape}  \vspace{-.2cm}\\
         \text{\citep{lee2022minimax}} 
    \end{array}$
   & $\displaystyle \sum_{i : \Delta_i >0} \Big(Ke^u + \log \Big(\frac{T\Delta_i^\frac{1+\epsilon}{\epsilon}}{K}\Big)^\frac{1+\epsilon}{\epsilon} \Big)\frac{1}{\Delta_i^{1/\epsilon}} $ & \xmark
   & $\displaystyle K^\frac{\epsilon}{1+\epsilon}u^\frac{1}{1+\epsilon}T^\frac{1}{1+\epsilon}e^u$ & \xmark
   & \xmark
   & \xmark
   & \cmark
   & \cmark
   & ---\\
   \hline
   $\begin{array}{l}
         \text{\htinf}  \vspace{-.2cm}\\
         \text{\citep{huang2022adaptive}} 
    \end{array}$
   & $\displaystyle\sum_{i: \Delta_i >0} \left(\frac{u}{\Delta_i}\right)^{1/\epsilon}\log T$ & \cmark
   & $ \displaystyle K^\frac{\epsilon}{1+\epsilon}u^\frac{1}{1+\epsilon}T^\frac{1}{1+\epsilon}$ & \cmark
   & \xmark
   & \xmark
   & \xmark
   & \xmark
   & Assumption \ref{ass:truncated_np}\\
   \hline
   $\begin{array}{l}
         \text{\opthtinf}  \vspace{-.2cm}\\
         \text{\citep{huang2022adaptive}} 
    \end{array}$
   & $\displaystyle\sum_{i: \Delta_i >0} \left(\frac{u^2}{\Delta_i^{2-\epsilon}}\right)^{1/\epsilon}\log T$ & \xmark
   & $\displaystyle K^\frac{\epsilon}{2}u^\frac{1}{1+\epsilon}T^\frac{2-\epsilon}{2}$ & \xmark
   & \cmark
   & \cmark
   & \cmark
   & \cmark
   & Assumption \ref{ass:truncated_np}\\
    \hline
    $\begin{array}{l}
         \text{\adatinf}  \vspace{-.2cm}\\
         \text{\citep{huang2022adaptive}} 
    \end{array}$
   & --- & ---
   & $\displaystyle K^\frac{\epsilon}{1+\epsilon}u^\frac{1}{1+\epsilon}T^\frac{1}{1+\epsilon}$ & \cmark
   & \cmark
   & \cmark
   & \cmark
   & \cmark
   & Assumption \ref{ass:truncated_np}\\
   \hline
    $\begin{array}{l}
         \text{\rucbtea $^\ddagger$}  \vspace{-.2cm}\\
         \text{\citep{ashutosh2021bandit}} 
    \end{array}$
   & $\displaystyle\sum_{i : \Delta_i >0} \frac{f(T)}{1-\frac{2}{\Delta_i \log f(t)}}\log T$ & \xmark
   & --- & ---
   & \cmark
   & \cmark
   & \cmark
   & \cmark
   & 
   $\begin{array}{c}
   T \text{ s.t. } \\ 3u\log f(T) < 2f(T)^\epsilon
   \end{array}$
   \\
   \hline
   $\begin{array}{l}
         \text{ \rucbmom $^\ddagger$}  \vspace{-.2cm}\\
         \text{\citep{ashutosh2021bandit}} 
    \end{array}$
   & $\displaystyle\sum_{i: \Delta_i >0} \Delta_i\left(\frac{2f(T)}{\Delta_i}\right)^\frac{1}{g(T)}\log T$ & \xmark
   & --- & ---
   & \cmark
   & \cmark
   & \cmark
   & \cmark
   & $T$ s.t. $\substack{ g(T) < \frac{\epsilon}{1+\epsilon} \\ f(T)>(12u)^\frac{1}{1+\epsilon}}$\\
   \hline
   \cellcolor{gray!15} $\begin{array}{l} \text{\algnameshort}\vspace{-.2cm} \\ \text{\textbf{(ours)}} \end{array}$
   &  \cellcolor{gray!15} $\displaystyle\sum_{i: \Delta_i >0} \left(\left(\frac{u}{\Delta_i}\right)^{1/\epsilon}+ \frac{\Delta_i}{\mathbb{P}_{\nu_i}(X\neq 0)}\right)\log T$ & \cellcolor{gray!15}  \cmark
   &  \cellcolor{gray!15} $\displaystyle
   \begin{array}{l}
         \displaystyle K^\frac{\epsilon}{1+\epsilon}u^\frac{1}{1+\epsilon}T^\frac{1}{1+\epsilon}(\log T)^\frac{\epsilon}{1+\epsilon} \\
         \displaystyle\quad + \sum_{i: \Delta_i>0} \frac{\Delta_i}{\mathbb{P}_{\nu_i}(X\neq 0)}\log T
   \end{array}
    $ & \cellcolor{gray!15}  \cmark
   &  \cellcolor{gray!15}\cmark
   &  \cellcolor{gray!15}\cmark
   &  \cellcolor{gray!15}\cmark
   &  \cellcolor{gray!15}\cmark
   & \cellcolor{gray!15} Assumption \ref{ass:truncated_np}\\
   \thickhline
  \end{tabular}}
  
  \vspace{-.2cm}\begin{flushleft}\footnotesize{$^\star$ The bound depends on the centered absolute moment $v \coloneqq \max_{i\in[K]}\mathbb{E}_{X \sim \nu_i}[|X-\mu_i|^{1+\epsilon}]$ of order $1+\epsilon$.   \\ $^\dagger$ $D_{\text{KL}}^{\text{inf}}(\eta, x) \coloneqq \inf\{D_\text{KL}(\eta, \kappa): \kappa \in \mathcal{P}_{\text{HT}}(\epsilon,u) \; \text{and}\; \mathbb{E}_{X \sim \kappa}[X] \ge x\}$. \\ $^\ddagger$ $f$ and $g$ are to be given in input. Choosing an optimal value of those would require knowing $\epsilon$ and $u$.\\
  {$^\mathsection$} Matching the instance-dependent lower bound for the non-adaptive case w.r.t. $T$, $1/\Delta_i$, $u$ (or $v$), and $\epsilon$, up to constants. \\
  {$^\mathparagraph$} Matching worst-case lower bound for the non-adaptive case w.r.t. $T$, $K$, $u$ (or $v$), and $\epsilon$, up to logarithmic terms.}\end{flushleft} 
 \centering\vspace{-.5cm}\caption{Comparison with the state-of-the-art. The regret bounds are deprived by constants.}
 \label{tab:comparison}
\end{table}

\cite{bubeck2013bandits} represents the most influential work in this area, formally introducing the setting, deriving both instance-dependent and worst-case lower bounds, and proposing the first non-adaptive algorithm, namely \robustucb. Such an algorithm can be instanced with three robust estimators: \emph{trimmed mean} (TM), \emph{median of means} (MoM), and \emph{Catoni estimator} (Catoni) achieving near-optimal regret guarantees from both instance-dependent and worst-case cases. The first minimax optimal algorithm was proposed in \cite{wei2020minimax}, namely \rmoss, removing the $(\log T)^\frac{\epsilon}{1+\epsilon}$. Instead, in \cite{agrawal2021regret}, the authors propose \klinfucb attaining an asymptotically-optimal instance-dependent upper bound, highlighting the dependence on the instance with the KL-divergence, similarly as \cite{garivier2011kl} for non-heavy-tailed bandits. These algorithms, however, require the knowledge of $\epsilon$ and $u$, \emph{i.e.,} they are non-adaptive.\footnote{The \emph{truncated mean} requires the knowledge of $\epsilon$ and $u$ in the construction of the expected reward estimator too.}

In \cite{ashutosh2021bandit}, the authors show that \emph{adaptivity} comes at a cost in both subgaussian and heavy-tailed bandits. In particular, logarithmic instance-dependent regret is unachievable when no further information on the environment is available. They introduce two algorithms, namely \rucbtea and \rucbmom, exploiting the TM and the MoM estimators, respectively. Although, in principle, they do not require the knowledge of $\epsilon$ or $u$ for execution, logarithmic regret cannot be achieved but only approached with arbitrary precision. Moreover, the bounds hold only for a learning horizon $T$ larger than a threshold depending on $\epsilon$ and $u$. No worst-case analysis is presented. 

The closest work to ours is \cite{huang2022adaptive} where the authors introduce the \textit{adversarial heavy-tailed bandits}, in which an adversary chooses HT distributions for the losses. They first introduce the \textit{truncated non-negativity} (analogous to our Assumption \ref{ass:truncated_np} for rewards) representing, to the best of authors' knowledge, the only assumption not explicitly related to $\epsilon$ and $u$. Three algorithms are provided: \htinf, \opthtinf, and \adatinf, all analyzed under this assumption. \htinf requires knowledge of both $\epsilon$ and $u$ and it is nearly optimal. Differently, both \opthtinf and \adatinf are $(\epsilon,u)$-adaptive. However, the instance-dependent bound of \opthtinf exposes an inconvenient dependence on $\big(\frac{u^2}{\Delta_i^{2-\epsilon}}\big)^{1/\epsilon}$ and the worst-case bound scales with $T^\frac{2-\epsilon}{2}$, both failing to match the lower bounds of the non-adaptive setting. Finally, the worst-case bound of \adatinf matches the non-adaptive lower bound, unfortunately, no instance-dependent analysis is provided.

Finally, \cite{lee2020optimal} introduce the \apetwo algorithm, adaptive in $u$, that, unfortunately, does not achieve logarithmic instance-dependent regret that, instead, scales with $T^\frac{1}{1+\epsilon}\log T$ and displays an inconvenient exponential dependence $e^u$. A modified version of this algorithm, namely \mrape, introduced in \cite{lee2022minimax}, succeeds in removing the polynomial dependence on $T$, now poly-logarithmic $(\log T)^{\frac{1+\epsilon}{\epsilon}}$, but maintains the dependence on $e^u$. 



\section{Minimax Lower Bounds for Adaptive Heavy-Tailed Bandits}
\label{sec:lower_bounds}
In this section, we address \ref{question1}, by analyzing the challenges of the $(\epsilon,u)$-adaptive regret minimization problem, \emph{i.e.,} without the knowledge of $\epsilon$ and $u$. We start by revising the minimax regret lower bound derived in \citet{bubeck2013bandits} for the \emph{non-adaptive} case, \emph{i.e.,} when $\epsilon$ and $u$ are known (Theorem~\ref{thr:standard_lb}). Then, in Section~\ref{sec:negative}, we provide two novel \emph{negative results} showing that achieving the same worst-case regret guarantees when either $u$ (Theorem~\ref{thr:u_adaptive_lb}) or $\epsilon$ (Theorem~\ref{thr:epsilon_adaptive_lb}) are unknown is not possible. Finally, in Section~\ref{sec:newLB}, we derive a new minimax regret lower bound under the \emph{truncated non-positivity} (Assumption~\ref{ass:truncated_np}), illustrating how, even in the non-adaptive case, this assumption does not lead to smaller regret lower bounds.\footnote{We remark that from the instance-dependent regret perspective, a negative answer to the possibility of achieving logarithmic regret with adaptive algorithms for heavy-tailed bandits has been already provided in~\citep[][Theorem 1]{ashutosh2021bandit}. Thus, our focus is on the minimax regret perspective.}

Let us start by recalling the minimax regret lower bound for the non-adaptive case.
\begin{restatable}[Minimax lower bound -- non-adaptive, \cite{bubeck2013bandits}]{thr}{StandardLowerBound}\label{thr:standard_lb}
Fix $\epsilon \in (0,1]$ and $u \ge 0$. For every algorithm \texttt{Alg}, sufficiently large learning horizon $T \in \mathbb{N}$, and number of arms $K \in \mathbb{N}_{\ge 2}$, it holds that:
\begin{equation}
    \label{eq:standard_lb}
   \sup_{\bm{\nu} \in \mathcal{P}_{\text{HT}}(\epsilon,u)^K} R_T(\texttt{Alg},\bm{\nu}) \ge c_0 K ^{\frac{\epsilon }{1+\epsilon}}(uT)^{\frac{1}{1+\epsilon}},
\end{equation}
where $c_0>0$ is a constant independent of $u$, $\epsilon$, $K$, and $T$.
\end{restatable}

This result shows how the dependency on $T$ deteriorates as $\epsilon$ approaches $0$ and, instead, when the variance is finite, \emph{i.e.,} $\epsilon = 1$, the lower bound displays the same order as the one for stochastic MABs with subgaussian rewards~\citep{lattimore2020bandit}.

\subsection{Negative Results about Adaptivity}\label{sec:negative}
We now move to our negative results about the possibility of matching the minimax regret lower bound of the non-adaptive setting using $(\epsilon,u)$-adaptive algorithms. The following result shows that any $u$-adaptive algorithm cannot achieve the same regret as in the non-adaptive case in  Theorem \ref{thr:standard_lb}.

\begin{restatable}[Minimax lower bound -- $u$-adaptive]{thr}{uAdaptiveLowerBound}\label{thr:u_adaptive_lb}
Fix $\epsilon \in (0,1]$. For every algorithm \texttt{Alg}, sufficiently large learning horizon $T \in \mathbb{N}$, and number of arms $K \in \mathbb{N}_{\ge 2}$, it holds that:
\begin{equation}\label{eq:adaptive_u_lb0}
    \sup_{u \ge 0}  \sup_{\bm{\nu} \in \mathcal{P}_{\text{HT}}(\epsilon,u)^K} \frac{R_T(\texttt{Alg},\bm{\nu})}{u^{\frac{1}{1+\epsilon}}} = +\infty.
\end{equation}
More precisely, for every $u'\ge u\ge 0$, under the same conditions above, there exist two instances $\bm{\nu} \in \mathcal{P}_{\text{HT}}(\epsilon,u)$ and $\bm{\nu}' \in \mathcal{P}_{\text{HT}}(\epsilon,u')$ such that:
\begin{equation}
    \label{eq:adaptive_u_lb}
    \max\left\{\frac{R_T(\texttt{Alg},\boldsymbol \nu)}{u^{\frac{1}{1+\epsilon}}}, \frac{R_T(\texttt{Alg},\boldsymbol \nu')}{(u')^{\frac{1}{1+\epsilon}}}\right\} \ge c_1 \left(\frac{u'}{u}\right)^{\frac{\epsilon}{(1+\epsilon)^2}} T^{\frac{1}{1+\epsilon}},
\end{equation}
where $c_1>0$ is a constant independent of $u$, $u'$, and $T$.
\end{restatable}

Some remarks are in order. First, let us observe that for proving the negative result, we have studied the ratio $\frac{R_T(\texttt{Alg},\boldsymbol \nu)}{u^{\frac{1}{1+\epsilon}}}$. Indeed, if there exists a $u$-adaptive regret minimization algorithm matching the lower bound for the non-adaptive case (Theorem~\ref{thr:standard_lb}), this ratio would not depend on $u$ anymore. It is convenient to start commenting on Theorem~\ref{thr:u_adaptive_lb} from the lower bound in Equation~\eqref{eq:adaptive_u_lb}. Here, we show the existence of two heavy-tailed bandit instances $\bm{\nu}$ and $\bm{\nu}'$, characterized by the same moment order $\epsilon$ but possibly different moment bounds  $u'\ge u$, for which any algorithm suffers (apart from constants)  a regret that preserves the dependence on $T$ but introduces a dependence on the ratio $u'/u$. Since we can make the ratio arbitrarily large, by varying $u,u' \in \mathbb{R}_{\ge 0}$, we conclude the statement in Equation~\eqref{eq:adaptive_u_lb0} showing that the mimimax lower bound degenerates to infinity. This shows that, with no additional assumptions, there exists no $u$-adaptive algorithm matching the lower bound for the non-adaptive case (Theorem~\ref{thr:standard_lb}).

We now present the counterpart negative result concerning adaptivity to the moment order $\epsilon$.
\begin{restatable}[Minimax lower bound -- $\epsilon$-adaptive]{thr}{epsAdaptiveLowerBound}\label{thr:epsilon_adaptive_lb}
Fix $u=1$. For every algorithm \texttt{Alg}, sufficiently large learning horizon $T \in \mathbb{N}$, and number of arms $K \in \mathbb{N}_{\ge 0}$, it holds that:
\begin{equation}\label{eq:adaptive_e_lb0}
    \sup_{\epsilon \in (0,1]}  \sup_{\bm{\nu} \in \mathcal{P}_{\text{HT}}(\epsilon,u)^K} \frac{R_T(\texttt{Alg},\bm{\nu})}{T^{\frac{1}{1+\epsilon}}} \ge c_2 T^{\frac{1}{16}}.
\end{equation}
More precisely, for every $\epsilon,\epsilon' \in (0,1]$ with $\epsilon'\le \epsilon$, under the same conditions above, there exist two instances $\bm{\nu} \in \mathcal{P}_{\text{HT}}(\epsilon,u)$ and $\bm{\nu}' \in \mathcal{P}_{\text{HT}}(\epsilon',u)$ such that:
\begin{equation}
    \label{eq:adaptive_e_lb}
    \max\left\{\frac{R_T(\texttt{Alg},\boldsymbol \nu)}{T^{\frac{1}{1+\epsilon}}}, \frac{R_T(\texttt{Alg},\boldsymbol \nu')}{T^{\frac{1}{1+\epsilon'}}}\right\} \ge c_2 T^{\frac{\epsilon'(\epsilon-\epsilon')}{(1+\epsilon)(1+\epsilon')^2}},
\end{equation}
where $c_2>0$ is a constant independent of $\epsilon$, $\epsilon'$, and $T$.
\end{restatable}

Differently from Theorem \ref{thr:u_adaptive_lb}, here we target the ratio $\frac{R_T(\texttt{Alg},\bm{\nu})}{T^{\frac{1}{1+\epsilon}}}$ for deriving the negative result. Indeed, having fixed $u=1$, if an $\epsilon$-adaptive algorithm exists matching the lower bound of Theorem~\ref{thr:standard_lb}, then, the considered ratio would not depend on $T$ anymore. Starting from the lower bound of Equation~\eqref{eq:adaptive_e_lb}, we observe that there exist two instances $\bm{\nu}$ and $\bm{\nu}'$, with $\epsilon$ and $\epsilon'$ as moment orders, for which the ratio is lower bounded by a function dependent on $T$. Since $\epsilon \ge \epsilon'$, the exponent to which $T$ is raised is non-negative and, consequently, the lower bound is a non-decreasing function of $T$. By letting $\epsilon$ and $\epsilon'$ range in $[0,1)$, we obtain the minimax bound of Equation~\eqref{eq:adaptive_e_lb0}, displaying a gap of order $T^{\frac{1}{16}}$, which is attained by taking $\epsilon=1$ and $\epsilon'=1/3$. This result shows that there exists no $\epsilon$-adaptive algorithm able to suffer the same regret as in the non-adaptive case of Theorem~\ref{thr:standard_lb}. 

Combining Theorem~\ref{thr:u_adaptive_lb} with Theorem~\ref{thr:epsilon_adaptive_lb}, we conclude the non-existence of an $(\epsilon,u)$-adaptive algorithm suffering the same regret guarantees as in the non-adaptive case. It is worth noting that the constructions employed for deriving the lower bounds presented in this section violate Assumption~\ref{ass:truncated_np}.

\subsection{Minimax Lower Bound under Assumption~\ref{ass:truncated_np}}\label{sec:newLB}
The results presented above show that, if our goal is to match the worst-case bound of the non-adaptive case of Theorem~\ref{thr:standard_lb}, we surely need to enforce additional assumptions. \citet{huang2022adaptive} succeeds in this task by using the \emph{truncated non-positivity} (Assumption~\ref{ass:truncated_np}, actually, its dual version for losses). We may wonder whether enforcing Assumption \ref{ass:truncated_np} radically simplifies the problem. In the following, we show that this is not the case, by deriving a novel minimax lower bound for the non-adaptive case under this assumption of the same order as that of Theorem~\ref{thr:standard_lb}.

\begin{restatable}[Minimax lower bound under Assumption~\ref{ass:truncated_np} - non-adaptive]{thr}{AssLowerBound}\label{thr:AssLowerBound}
Fix $\epsilon \in (0,1]$ and $u \ge 0$. For every algorithm \texttt{Alg}, sufficiently large learning horizon $T \in \mathbb{N}$, and number of arms $K \in\mathbb{N}_{\ge 2}$, it holds that:
\begin{equation}
    \label{eq:standard_lb}
   \sup_{\substack{\bm{\nu} \in \mathcal{P}_{\text{HT}}(\epsilon,u)^K \\ \bm{\nu} ~\text{fulfills Assumption~\ref{ass:truncated_np}}}} R_T(\texttt{Alg},\bm{\nu}) \ge c_3 K ^{\frac{\epsilon }{1+\epsilon}}(uT)^{\frac{1}{1+\epsilon}},
\end{equation}
where $c_3>0$ is a constant independent of $u$, $\epsilon$, $K$ and $T$.
\end{restatable}

Since ($i$) achieving the regret of Theorem~\ref{thr:standard_lb} without further assumptions is not possible (Theorems~\ref{thr:u_adaptive_lb}~\ref{thr:epsilon_adaptive_lb}) and ($ii$) Assumption~\ref{ass:truncated_np} does not change the complexity of the non-adaptive case (Theorem~\ref{thr:AssLowerBound}), it makes sense to search for adaptive algorithms matching Theorem~\ref{thr:standard_lb} under Assumption~\ref{ass:truncated_np}.

\section{Trimmed Mean Estimator with Empirical Threshold}\label{sec:tmest}
In this section, we present our novel \emph{trimmed mean with empirical threshold} estimator, in which the threshold is computed from data.
The trimmed mean estimator~\citep{bickel1965some}, common in heavy-tailed statistics, cuts off the observations outside a predefined interval $[-M, M]$ with $M \ge 0$, named \emph{trimming threshold}. Given a set of $s\in\mathbb{N}_{\ge 1}$ i.i.d. random variables $\mathbf{X} = \{X_1, \ldots, X_s\}$, with expected value $\mu \coloneqq \E[X_1]$,  the trimmed mean estimator with threshold $M$ is defined as:
\begin{equation}
    \label{eq:trimmed_mean}
    \widehat{\mu}_s(\mathbf{X};M) \coloneqq \frac{1}{s}\sum_{j\in[s]} X_j \mathbbm{1}_{\{|X_j|\le M\}}.
\end{equation}
The following result shows that, under truncated non-positivity (Assumption~\ref{ass:truncated_np}), it is possible to design an upper confidence bound on $\mu$ based on the trimmed mean estimator $\widehat{\mu}_s(\mathbf{X};M)$ that can be computed with no knowledge of $\epsilon$ and $u$, depending only on the trimming threshold $M$.

\begin{restatable}[$(\epsilon,u)$-free Upper Confidence Bound]{lemma}{thrConcentrationMGeneric}\label{thr:concentration_ineq}
Let $\delta \in (0,1/2)$ and $\mathbf{X} = \{X_1, \ldots, X_s\}$ be a set of $s \in \mathbb{N}_{\ge 2}$ i.i.d. random variables satisfying $X_1 \sim \nu \in \mathcal{P}_{\text{HT}}(\epsilon,u)$, $\mu \coloneqq \E[X_1]$, and $M >0$ be a (possibly random) trimming threshold independent of $\mathbf{X}$. Then, under Assumption \ref{ass:truncated_np}, it holds that:
\begin{equation}
    \label{eq:concentration_ineq}
    \mathbb{P} \Big( \mu - \widehat{\mu}_s(\mathbf{X};M) \leq \sqrt{\frac{2V_s(\mathbf{X};M) \log \delta^{-1}}{s}}+\frac{ 10 M \log \delta^{-1}}{s} \Big) \geq 1 - 2\delta,
\end{equation}
where $V_s(\mathbf{X};M)$ is the sample variance of the trimmed random variables, defined as:
\begin{align}\label{eq:sampleVariance}
    V_s(\mathbf{X};M) \coloneqq \frac{1}{s-1} \sum_{j \in [s]} (X_j \mathbbm{1}_{\{|X_j| \le M\}} - \widehat{\mu}_s(\mathbf{X};M))^2.
\end{align}
\end{restatable}

The result is obtained by applying the \emph{empirical Bernstein's inequality}~\citep{maurer2009empirical} and it is a \emph{one-sided inequality} because of the nature of Assumption~\ref{ass:truncated_np}. From an algorithmic perspective, this enables us to build an optimistic index that does not require knowing the values of $\epsilon$ and $u$ and represents the essential role of Assumption~\ref{ass:truncated_np} in our \algnameshort algorithm.

The next step consists of computing the trimming threshold $M$ in a fully data-driven way. Notice that the trimming threshold in \robustucb~\citep{bubeck2013bandits} is selected thanks to the knowledge of $\epsilon$ and $u$ as $\widetilde{M}_{s}(\delta) = \big( \frac{us}{\log \delta^{-1}} \big)^{\frac{1}{1+\epsilon}} $. Instead, we follow a procedure similar to that of~\cite{wang2021new} for Huber regression, and we estimate an \emph{empirical trimming threshold} via a root-finding problem. Specifically, given a set of $s \in \mathbb{N}_{\ge 1}$ i.i.d. random variables $\mathbf{X}' = \{X_1', \ldots, X_{s}'\}$ (independent of $\mathbf{X}$), the empirical threshold $\widehat{M}_s(\delta)$ is the solution of the equation:\footnote{The sets $\mathbf{X}$ and $\mathbf{X}'$ are chosen to have the same cardinality $s$ to provide more readable results.}
\begin{align}\label{eq:M_hat}
    f_{s}(\mathbf{X}';M,\delta) \coloneqq \frac{1}{s} \sum_{j\in[s]}\frac{\min \{(X'_j)^{2}, M^{2}\}}{M^2} - \frac{c \log \delta^{-1}}{s} = 0,
\end{align}
where $c>0$ is a hyperparameter that will be set later. If the number of non-zero samples $X_j'$ is sufficiently large, \emph{i.e.,} $\sum_{j\in [s]} \mathbbm{1}_{\{X_j' \neq 0\}} > c \log \delta^{-1}$ (see Proposition~\ref{thr:uniqueness} for details), Equation~\eqref{eq:M_hat} admits a unique positive solution, that we denote as $\widehat{M}_s(\delta)$.\footnote{An efficient algorithm for solving Equation~\eqref{eq:M_hat} and its computational complexity analysis are reported in Appendix~\ref{apx:comp}.}
A reader might notice that we are solving the ``sample version'' of the ``population version'' equation $\mathbb{E}[f_{s}(\mathbf{X}';M,\delta)] = 0$. Denoting with  ${M}_s(\delta)$ the solution (when it exists) of this latter equation, we can establish a meaningful relation between ${M}_s(\delta)$ and the threshold $\widetilde{M}_s(\delta)$ used by \robustucb~\citep{bubeck2013bandits}:
\begin{align}
    c \log \delta^{-1} & = \mathbb{E}[\min\{(X_1')^2/{M}_{s}(\delta)^2, 1\}] \le \mathbb{E}[|X_1'|^{1+\epsilon}] {M}_{s}(\delta)^{-1-\epsilon} \\
    & \implies  {M}_{s}(\delta) \le \Big( \frac{us}{c \log \delta^{-1}} \Big)^{\frac{1}{1+\epsilon}} = c^{-\frac{1}{1+\epsilon}} \widetilde{M}_{s}(\delta).
\end{align}
In practice, however, we cannot solve the ``population'' equation $\mathbb{E}[f_{s}(\mathbf{X}';M,\delta)] = 0$ and we need to resort to the ``sample version'', delivering $\widehat{M}_s(\delta)$. The following result shows that $\widehat{M}_s(\delta)$ behaves (in high probability) analogously to ${M}_s(\delta)$, for a suitable choice of $c$.
\begin{restatable}[Bounds on $\widehat {M}_{s}(\delta)$]{thr}{boundsonM}
\label{thr:M_hat_bound}
Let $\delta \in (0,1/2)$ and $\mathbf{X}' = \{X_1', \ldots, X_s'\}$ be a set of $s \in \mathbb{N}_{\ge 1}$ i.i.d. random variables satisfying $X_1' \sim \nu \in \mathcal{P}_{\text{HT}}(\epsilon,u)$, and let $\widehat{M}_{s}(\delta)$ be the (random) positive root of Equation~\eqref{eq:M_hat} with $c > 2$. Then, if  $\widehat {M}_{s}(\delta)$ exists, with probability at least $1-2\delta$, it holds that:
\begin{align}
    \widehat {M}_{s}(\delta) \leq \bigg(\dfrac{us}{(\sqrt{c}-\sqrt{2})^2\log \delta^{-1}}\bigg)^{\frac{1}{1+\epsilon}} \quad \text{and} \quad \mathbb{P}\left(|X_1| > \widehat {M}_{s}(\delta) \right) \le (\sqrt{c}+\sqrt{2})^2 \frac{\log \delta^{-1}}{s}. \label{eq:ineqM_1}
\end{align}
\end{restatable}
The proof relies on the concentration inequalities for \emph{self-bounding functions} \citep{maurer2006concentration, maurer2009empirical}. By selecting $c > (1+\sqrt{2})^2$, we have that, with probability $1-2\delta$, the empirical threshold $\widehat {M}_{s}(\delta)$ is smaller than $\widetilde {M}_{s}(\delta)$, used in \robustucb. Furthermore, in \citep{bubeck2013bandits}, the particular form of the \emph{deterministic} threshold $\widetilde {M}_{s}(\delta)$ allows the authors to apply \emph{Bernstein's inequality} and obtain a concentration bound explicitly depending on $\epsilon$ and $u$ (Lemma 1):
\begin{equation}
\label{eq:concentarion_ineq_no_M}
    \mathbb{P}\bigg( \left|\widehat{\mu}_s(\mathbf{X};\widetilde {M}_{s}(\delta)) - \mu \right| \le 4u^\frac{1}{1+\epsilon}\left(\frac{\log \delta^{-1}}{s}\right)^\frac{\epsilon}{1+\epsilon}\bigg) \ge 1-2\delta,
\end{equation}
We now show that using the \emph{random} threshold $\widehat {M}_{s}(\delta)$, instead, still allows achieving analogous guarantees with just a slightly larger constant.
\begin{restatable}[$(\epsilon,u)$-dependent Concentration Bound]{thr}{TheoremConc}\label{thr:conc_ineq_2}
Let $\delta \in (0,1/4)$, $\mathbf{X} = \{X_1, \ldots, X_{s/2}\}$, and $\mathbf{X}' = \{X_1', \ldots, X_{s/2}'\}$ be two independent sets of $s/2 \in \mathbb{N}_{\ge 2}$ i.i.d. random variables satisfying $X_1 \sim \nu \in \mathcal{P}_{\text{HT}}(\epsilon,u)$, $\mu \coloneqq \E[X_1]$, and let $\widehat{M}_{s}(\delta)$ be the (random) positive root of Equation~\eqref{eq:M_hat} with $c = (1+\sqrt{2})^2$. Then, if  $\widehat {M}_{s}(\delta)$ exists, it holds that:
\begin{equation}
    \label{eq:concentration_ineq2}
    \mathbb{P} \bigg( \left|\widehat{\mu}_s(\mathbf{X};\widehat{M}_{s}(\delta)) - \mu\right| \leq 8 u^{\frac{1}{1+\epsilon}} \left(\frac{\log \delta^{-1}}{s} \right)^{\frac{\epsilon}{1+\epsilon}} \bigg) \geq 1 - 4\delta.
\end{equation}
\end{restatable}

First, we notice that, at the price of a slightly larger constant, this concentration bound displays the same behavior as that of Equation~\eqref{eq:concentarion_ineq_no_M}. However, remarkably, our estimator is fully data-driven as the trimming threshold $\widehat{M}_{s}(\delta)$ is computed from dataset $\mathbf{X}'$ (\emph{i.e.,} half of the available samples) with no knowledge of $\epsilon$ and $u$. Second, differently from Lemma~\ref{thr:concentration_ineq}, this is a \emph{double-sided inequality} and holds even without Assumption~\ref{ass:truncated_np}.
From a technical perspective, this result is proved by resorting to \emph{Bernstein's inequality} and Theorem~\ref{thr:M_hat_bound} to control the values of the estimated threshold $\widehat{M}_{s}(\delta)$.

Summarizing, we conclude that the trimmed mean estimator $\widehat{\mu}_s(\mathbf{X};\widehat{M}_{s}(\delta))$ with the empirical threshold $\widehat{M}_{s}(\delta)$ fulfills two important properties: ($i$) under Assumption~\ref{ass:truncated_np}, it enjoys an upper confidence bound that is fully empirical and $(\epsilon,u)$-free (Lemma~\ref{thr:concentration_ineq}). This bound will be used in the implementation of the \algnameshort algorithm; ($ii$) it enjoys (up to constants) the same concentration properties as the truncated mean with the $(\epsilon,u)$-dependent threshold $\widetilde{M}_s(\delta)$ (Theorem~\ref{thr:conc_ineq_2}). This bound, instead, will be used in the analysis of the \algnameshort algorithm.



\section{An $(\epsilon,u)$-Adaptive Approach for Heavy-Tailed Bandits}
\label{sec:algo}
In this section, we address \ref{question2} 
by presenting \algname (\algnameshort, Algorithm \ref{alg:alg}), an $(\epsilon,u)$-adaptive \emph{anytime} regret minimization algorithm able to operate in the heavy-tailed bandit problem \textit{with no prior knowledge on $\epsilon$ or $u$}, and providing its regret analysis.

\subsection{The \algnameshort Algorithm}
 \algnameshort (Algorithm \ref{alg:alg}) is based on the optimism in-the-face-of-uncertainty principle, and built upon the \robustucb strategy from~\cite{bubeck2013bandits} leveraging the estimator presented in Section~\ref{sec:tmest}. \algnameshort keeps track of the number of times every arm $i \in [K]$ has been selected $N_i(\tau)$ and maintains two disjoint sets of rewards  $\mathbf{X}_i(\tau)$ and $\mathbf{X}_i'(\tau)$ (line~\ref{line:init}). Specifically, $\mathbf{X}_i'(\tau)$ will be employed to compute the empirical threshold, while $\mathbf{X}_i(\tau)$ for the trimmed mean estimator. The algorithm operates over $\lfloor T/2 \rfloor$ rounds, indexed by $\tau$, and, in every round $\tau \in \lfloor T/2 \rfloor$, it collects \emph{two} samples from the selected arm $I_\tau$ (the time index is $t=2\tau$). Specifically, \algnameshort first computes the \emph{upper confidence bound} index $B_i(\tau)$ for every arm $i \in [K]$. If the condition for the existence of the positive root of Equation~\eqref{eq:M_hat} is not verified (line~\ref{line:cond}), the index $B_i(\tau)$ is set to $+\infty$ (line~\ref{line:round-robin}), forcing the algorithm to pull arm $i$. Instead, if the condition is verified,  the empirical threshold is computed $\textcolor{blue}{\widehat{M}_{i}(\tau)} \leftarrow \widehat{M}_{N_i(\tau-1)}(\tau^{-3})$ (line~\ref{line:M_hatAlg}) according to Equation~\eqref{eq:M_hat} with $c=(1+\sqrt{2})^2$ using the dataset $\mathbf{X}_i'(\tau-1)$ and selecting $\delta=\tau^{-3}$. Then, the algorithm employs it (line~\ref{line:mean}) to compute the trimmed mean estimator $\textcolor{blue}{\widehat{\mu}_{i}(\tau)} \leftarrow \widehat{\mu}_{N_{i}(\tau-1)}(\mathbf{X}_i(\tau-1);\textcolor{blue}{\widehat{M}_{i}(\tau)})$ (as in Equation~\ref{eq:trimmed_mean}) and variance estimator $\textcolor{blue}{V_i(\tau)} \leftarrow {V}_{N_{i}(\tau-1)}(\mathbf{X}_i(\tau-1);\textcolor{blue}{\widehat{M}_{i}(\tau)})$ (as in Equation~\ref{eq:sampleVariance}) using the samples from the other dataset $\mathbf{X}_i(\tau-1)$. These quantities are then employed for the optimistic index computation $B_i(\tau)$ (line~\ref{line:ucbComp}) according to the empirical bound of Lemma~\ref{thr:concentration_ineq}. The optimistic arm $I_\tau$ is then played \emph{twice} (line~\ref{line:ucb}) and the two collected samples are used to augment the reward sets $\mathbf{X}_i(\tau)$ and $\mathbf{X}_i'(\tau)$, respectively (lines~\ref{line:u1}-\ref{line:u2}), and the arm pull counters $N_i(\tau)$ (line~\ref{line:u3}).

\RestyleAlgo{ruled}
\LinesNumbered
\begin{algorithm2e}[t]
\caption{\algname (\algnameshort).}\label{alg:alg}
\DontPrintSemicolon
{\small
\SetKwInOut{Input}{Input}

Initialize counters $N_{i}(0)= 0$, reward sets $\mathbf{X}_{i}(0) = \{\}$, $\mathbf{X}_{i}'(0) =\{\}$ for  every $i \in [K]$, $\tau \leftarrow 1$, $t \leftarrow 2\tau$ \label{line:init}\\
\While{$\tau \le \lfloor T/2 \rfloor$}{
    \For{$i \in [K]$}{
        \uIf{$N_{i}(\tau-1)=0$ \textbf{\emph{or}} $\sum_{X' \in \mathbf{X}_{i}'(\tau-1)}  \mathbbm{1}_{\{X' \neq 0\}} \le 4 \log\tau^{-3}$\label{line:cond}}
        {
            Compute the optimistic index: $B_{i}(\tau) = + \infty$ \label{line:round-robin}
        }
        \Else
        {
        Compute the trimming threshold: $\textcolor{blue}{\widehat{M}_{i}(\tau)} \leftarrow \widehat{M}_{N_i(\tau-1)}(\tau^{-3})$ solving the equation $f(\mathbf{X}_i(\tau-1);M,\tau^{-3}) = 0$ (Eq.~\ref{eq:M_hat}  with $c=(1+\sqrt{2})^2$) \label{line:M_hatAlg} 
        
        Compute the trimmed mean estimator: $\textcolor{blue}{\widehat{\mu}_{i}(\tau)} \leftarrow \widehat{\mu}_{N_{i}(\tau-1)}(\mathbf{X}_i(\tau-1);\textcolor{blue}{\widehat{M}_{i}(\tau)})$ (Eq.~\ref{eq:trimmed_mean}) and the variance estimator $\textcolor{blue}{V_i(\tau)} \leftarrow {V}_{N_{i}(\tau-1)}(\mathbf{X}_i(\tau-1);\textcolor{blue}{\widehat{M}_{i}(\tau)})$ (Eq.~\ref{eq:sampleVariance})\label{line:mean}\\
         Compute the optimistic index:\label{line:ucbComp} 
         \vspace{-.2cm}
         $$
            B_{i}(\tau) = \textcolor{blue}{\widehat{\mu}_{i}(\tau)} + \sqrt{\frac{\textcolor{blue}{2{V}_{i}(\tau)} \log \tau^3}{N_{i}(\tau-1)}}+ \frac{ 10\textcolor{blue}{\widehat{M}_{i}(\tau)} \log \tau^3}{N_{i}(\tau-1)}
         $$
         \vspace{-.2cm}
        }
    }
    Select arm ${I}_\tau \in \argmax_{i \in [K]} B_{i}(\tau)$, play it \textbf{twice}, and receive rewards $X$ and $X'$ \label{line:ucb} \\
    Update reward sets $\mathbf{X}_{I_\tau}(\tau) = \mathbf{X}_{I_\tau}(\tau-1) \cup \{X\}$, $\mathbf{X}_{i}(\tau) = \mathbf{X}_{i}(\tau-1)$ for every $i \neq I_\tau$ \label{line:u1} \\
     Update reward sets  $\mathbf{X}_{I_\tau}'(\tau) = \mathbf{X}_{I_\tau}'(\tau-1)\cup \{X'\}$, $\mathbf{X}_{i}'(\tau) = \mathbf{X}_{i}'(\tau-1)$ for every $i \neq I_\tau$  \label{line:u2} \\
    Update counters $N_{i}(\tau) = |\mathbf{X}_i(\tau)|$ for every $i \in [K]$, $\tau \leftarrow \tau +1$, $t \leftarrow 2\tau$\label{line:u3}\\
}
}
\end{algorithm2e}


\subsection{Regret Analysis}
In this section, we provide the regret analysis of \algnameshort under the truncated non-positivity assumption (Assumption~\ref{ass:truncated_np}). We start with the instance-dependent regret bound.
\begin{restatable}[Instance-Dependent Regret bound of \algnameshort]{thr}{TheoremID}\label{thr:upper_bound}
Let $\bm{\nu} \in \mathcal{P}_{\text{HT}}(\epsilon,u)^K$ and $T \in \mathbb{N}_{\ge 2}$ be the learning horizon. Under Assumption~\ref{ass:truncated_np}, \algnameshort suffers a regret bounded as:
\begin{equation}
    \label{eq:upper_bound}
   R_T (\text{\algnameshort}, \boldsymbol{\nu})\le \sum_{i:\Delta_i>0} \bigg[\bigg(120\bigg(\frac{u}{\Delta_i}\bigg)^\frac{1}{\epsilon}+\frac{24\Delta_i}{\mathbb{P}_{\nu_i}(X\neq 0)}\bigg)\log \frac{T}{2} + 20\Delta_i \bigg].
\end{equation}
\end{restatable}
Some observations are in order. We notice that the dependence on $\epsilon$ and $u$ match the instance-dependent lower bound for the non-adaptive case~\citep{bubeck2013bandits}. Note that the trimming threshold estimation requires in \algnameshort the \emph{forced exploration} (line~\ref{line:round-robin}) and leads to the additional logarithmic term $\sum_{i:\Delta_i>0}\frac{24\Delta_i}{\mathbb{P}_{\nu_i}(X\neq 0)} \log \frac{T}{2}$ that grows proportionally to the suboptimality gap $\Delta_i$ and inversely with the probability $\mathbb{P}_{\nu_i}(X\neq 0)$ of sampling a non-zero reward. This is explained by the condition (line~\ref{line:cond}) for the existence of a positive trimming threshold that requires a sufficiently large number of non-zero rewards. It is worth noting that for \emph{absolutely continuous} reward distributions, \emph{i.e.,} the ones we are interested in the heavy-tail setting, we have $\mathbb{P}_{\nu_i}(X\neq 0)=1$. In such a case, this additional regret term reduces to $24 \sum_{i:\Delta_i>0} \Delta_i \log \frac{T}{2}$, a term that was present in the regret bound of \robustucb with the Catoni estimator too.\footnote{We remark that this second term, although logarithmic, does not depend on the reciprocal of the suboptimality gaps and it is negligible when $1/\Delta_i \gg 1$ compared to the first one.} In general, we are currently unsure whether this term is unavoidable or an artifact of our algorithm and/or analysis.
From a technical perspective, the proof of Theorem \ref{thr:upper_bound} follows similar steps to the result provided by \cite{bubeck2013bandits} concerning the upper bound on regret for \robustucb, although additional care is needed to control simultaneously the concentration of the empirical threshold and of the trimmed mean estimator. In conclusion, this result provides a positive answer to our \ref{question2}, showing how \algnameshort nearly matches the instance-dependent lower bound for the non-adaptive case.


Finally, to complement the analysis, we provide the worst-case regret bound for \algnameshort.
\begin{restatable}[Worst-Case Regret bound of \algnameshort]{thr}{TheoremWC}\label{thr:inst_indep_regret}
Let $\bm{\nu} \in \mathcal{P}_{\text{HT}}(\epsilon,u)^K$ and $T \in \mathbb{N}_{\ge 2}$ be the learning horizon. Under Assumption~\ref{ass:truncated_np}, \algnameshort suffers a regret bounded as:
\begin{equation*}
    \label{eq:inst_indep_upperbound}
    \begin{aligned}
        R_{T}(\text{\algnameshort},\boldsymbol{\nu}) & \le  46 \bigg( K\log \frac{T}{2}\bigg)^{\frac{\epsilon}{1+\epsilon}} (uT)^{\frac{1}{1+\epsilon}}+ \sum_{i: \Delta_i >0} \bigg( \frac{24\Delta_i}{\mathbb{P}_{\nu_i}(X\neq 0)}\log \frac{T}{2} + 20\Delta_i\bigg) .
    \end{aligned}
\end{equation*} 
\end{restatable}
This result matches the lower bound from \cite{bubeck2013bandits}, up to logarithmic terms. 
\section{Conclusions}
In this paper, we studied the $(\epsilon,u)$-\textit{adaptive} heavy-tailed bandit problem, where no information on moments of the reward distribution, not even which of them are finite, is provided to the learner. Focusing on two appealing research questions, we have: ($i$) shown that, with no further assumptions, no adaptive algorithm can achieve the same worst-case regret guarantees as in the non-adaptive case; ($ii$) devised a novel algorithm (\algnameshort), based on a fully data-driven estimator, enjoying nearly optimal instance-dependent and worst-case regret, under the truncated non-positivity assumption.

Future directions include: ($i$) investigating the role of the truncated non-positivity assumption, especially, whether weaker assumptions can be formulated; ($ii$) characterizing the \textit{limits of $\epsilon$-adaptivity}, \ie the best performance attainable by an $\epsilon$-adaptive algorithm \emph{without additional assumption}; ($iii$) understanding whether the forced exploration for the empirical threshold computation in \algnameshort (and the corresponding regret term) is unavoidable.


\bibliography{bibliography}

\appendix

\clearpage
\section{Additional Related Works}
\label{apx:add_related}
In this section, we provide additional related works concerning adaptivity in statistics via Lepskii method and adaptivity in the case of subgaussian bandits.

\subsection{Adaptivity via Lepskii Method}
In \cite{bhatt2022nearly}, authors provide a novel technique to extend Catoni's M-estimator \citep{catoni2012challenging} to the infinite variance setting. In principle, their procedure relies on the knowledge of both $\epsilon$ and the centered moment $v$, however, they propose a strategy based on the Lepskii method \citep{lepskii1992asymptotically} to adapt to unknown $v$. While the Lepskii method is a popular choice in the adaptive statistics literature, we point out how it requires an upper bound on the quantity to estimate. Indeed, this method can be safely applied when adapting to unknown $\epsilon$ (since it can be at most $1$), but when it comes to $u$ (or the centered moment $v$), requiring an upper bound makes the approach \textit{not} fully adaptive.

\subsection{Adaptivity in Subgaussian Bandits}
In the literature of subgaussian stochastic bandits, $\sigma$ (the subgaussian proxy) is usually assumed to be known by the agent. However, many works consider settings in which this quantity is unknown. In this section, we discuss standard approaches to adapt to $\sigma$ (or estimate it) in subgaussian bandits, and show the additional difficulties implied by the heavy-tailed setting.

The main difference between $\sigma$ and $u$ is that the former can be estimated from data while guaranteeing strong convergence properties. In \cite{audibert2009exploration}, authors introduce \ucbv, a variation of the well-known \ucbone algorithm capable of using a data-driven estimation of the variance while keeping optimal performance. As customary in most of the literature, rewards are assumed to be bounded in a known range. However, in heavy-tailed bandits, it is not possible to make such an assumption, and the estimation of $u$ cannot be carried on.  Other works try to relax the assumption of bounded rewards by the means of other assumptions, \emph{e.g.,} a known upper bound on kurtosis \citep{lattimore2017scale}, or Gaussian rewards \citep{cowan2018normal}. 

Without additional assumptions, dealing with both the unknown range of the rewards and unknown $\sigma$ comes at a cost. As shown in \cite{hadiji2023adaptation}, when the range of the rewards is unknown and no additional knowledge on the distributions is available, it is impossible to be simultaneously optimal in both the instance-dependent sense and the worst-case one. The existence of such a trade-off shows how difficult is, even in subgaussian bandits, to attain optimal performances when no knowledge is given on the environment. As a consequence, also in fully adaptive heavy-tailed bandits, such an impossibility result holds. However, as we have discussed, thanks to a specific assumption not involving $\epsilon$ nor $u$ we can provide optimal regret guarantees in both cases.

\clearpage
\section{Proofs and Derivations}
\label{apx:proofs}
In this section, we prove the main theoretical results outlined in the paper.

\subsection{Lower Bounds}
\label{apx:lower-bounds}

\uAdaptiveLowerBound*
\begin{proof}
We start by constructing two heavy-tailed bandit instances with a common maximum order of moment $\epsilon$, but where $u'\ge u$. We use $\delta_x$ to denote the Dirac delta distribution centered on $x$
\paragraph{Base instance}
\begin{equation}
    \bm{\nu} = \begin{cases}
        &\nu_1 = \delta_0,\\
        &\nu_2 =  \left( 1-\Delta^{1+\frac{1}{\epsilon}} u^{-\frac{1}{\epsilon}} \right) \delta_0 + \Delta^{1+\frac{1}{\epsilon}} u^{-\frac{1}{\epsilon}}\delta_{u^{\frac{1}{\epsilon}} \Delta^{-\frac{1}{\epsilon}}},
    \end{cases}
\end{equation}
where $\Delta \in (0, u^{\frac{1}{1+\epsilon}})$.
Thus, we have $\mu_1 = 0$ and $\mu_2 = \Delta$. Furthermore, $\mathbb{E}_{X\sim \nu_1}[|X|^{1+\epsilon}] = 0$ and $\mathbb{E}_{X\sim\nu_2}[|X|^{1+\epsilon}] = u$
Therefore, the optimal arm is arm $2$ and $\bm{\nu} \in\mathcal{P}(\epsilon,u)^2$.
\paragraph{Alternative instance}
\begin{equation}
   \bm{\nu}' = \begin{cases}
        &\nu_1' =  \left(1 - (2\Delta)^{1+\frac{1}{\epsilon}} (u')^{-\frac{1}{\epsilon}}\right) \delta_0 + (2\Delta)^{1+\frac{1}{\epsilon}} (u')^{-\frac{1}{\epsilon}}\delta_{(u')^{\frac{1}{\epsilon}} (2\Delta)^{-\frac{1}{\epsilon}}}, \\
        &\nu_2' = \nu_2,
    \end{cases}
\end{equation}
where $\Delta \in (0, \frac{1}{2}(u')^{\frac{1}{1+\epsilon}})$. Thus we have $\mu_1' = 2\Delta$ and $\mu_2'=\Delta$. Furthermore, $\mathbb{E}_{X\sim\nu_1'}[|X|^{1+\epsilon}] = u'$ and $\mathbb{E}_{X\sim\nu_2'}[|X|^{1+\epsilon}]=u$. Therefore, the optimal arm is arm $1$ and $\bm{\nu} \in\mathcal{P}(\epsilon,u')^2$. 

We seek to prove that for any algorithm \texttt{Alg}, it holds that:
\begin{align*}
     \max \left\{ \frac{R_T(\texttt{Alg},\bm{\nu})}{(uT)^{\frac{1}{1+\epsilon}}},  \frac{R_T(\texttt{Alg},\bm{\nu}')}{(u'T)^{\frac{1}{1+\epsilon}}}\right\} \ge f(T,\epsilon,u,u'),
\end{align*}
being $f$ a function increasing in $T$. The proof merges the approach of~\citep[][Theorem 5]{bubeck2013bounded} with that of~\citep[][Chapters 14.2, 14.3]{lattimore2020bandit}.

First, we observe that:
\begin{align}
\label{eq:proof_lb_u_1}
     \max \left\{ \frac{R_T(\texttt{Alg},\bm{\nu})}{(uT)^{\frac{1}{1+\epsilon}}},  \frac{R_T(\texttt{Alg},\bm{\nu}')}{(u'T)^{\frac{1}{1+\epsilon}}}\right\} \ge  \frac{R_T(\texttt{Alg},\bm{\nu})}{(uT)^{\frac{1}{1+\epsilon}}} = \frac{\Delta \mathbb{E}_{\texttt{Alg},\bm{\nu}}[N_1(T)]}{(uT)^{\frac{1}{1+\epsilon}}}, 
\end{align}
where $\mathbb{E}_{\texttt{Alg},\bm{\nu}}[N_1(T)]$ is the expected number of times arm $1$ is pulled over the horizon $T$. Second, recalling which are the optimal arms in the two instances and that $u'\ge u$, we have:
\begin{equation}
\begin{aligned} \label{eq:proof_lb_u_2}
    & \max \left\{ \frac{R_T(\texttt{Alg},\bm{\nu})}{(uT)^{\frac{1}{1+\epsilon}}},  \frac{R_T(\texttt{Alg},\bm{\nu}')}{(u'T)^{\frac{1}{1+\epsilon}}}\right\} \ge \\
    & \hspace{2cm} \ge (u'T)^{-\frac{1}{\epsilon+1}} \frac{\Delta T}{2}  \max \left\{\mathbb{P}_{\texttt{Alg},\bm{\nu}}\left( N_1(T) \ge T/2\right) , \mathbb{P}_{\texttt{Alg},\bm{\nu}'}\left( N_1(T) < T/2\right)  \right\} \\
    & \hspace{2cm} \ge \frac{\Delta}{4} (u')^{-\frac{1}{\epsilon+1}} T^{\frac{\epsilon}{\epsilon+1}} \left( \mathbb{P}_{\texttt{Alg},\bm{\nu}}\left( N_1(T) \ge T/2\right) +  \mathbb{P}_{\texttt{Alg},\bm{\nu}'}\left( N_1(T) < T/2\right) \right) \\
    & \hspace{2cm} \ge \frac{\Delta}{8} (u')^{-\frac{1}{\epsilon+1}}  T^{\frac{\epsilon}{\epsilon+1}}  \exp \left( - \mathbb{E}_{\texttt{Alg},\bm{\nu}}[N_1(T)] D_{\text{KL}}(\nu_1 \| \nu_1') \right).
\end{aligned}
\end{equation}
where we used Bretagnolle-Huber inequality and divergence decomposition, together with $\max\{a,b\} \ge \frac{1}{2}(a+b)$ for $a,b\ge 0$. Let us now compute the KL-divergence, noting that $\nu_1 
\ll \nu_1'$:
\begin{equation}
\begin{aligned} \label{eq:proof_lb_u_3}
    D_{\text{KL}}(\nu_1 \| \nu_1') & = \nu_1(0) \log \frac{\nu_1(0)}{\nu_1'(0)} \\
    & = \log \frac{1}{ 1-(2\Delta)^{1+\frac{1}{\epsilon}} (u')^{-\frac{1}{\epsilon}}} \le c (2\Delta)^{1+\frac{1}{\epsilon}} (u')^{-\frac{1}{\epsilon}},
\end{aligned}
\end{equation}
for $\Delta \in (0, \left(\frac{1}{2}\right)^{\frac{2\epsilon+1}{1+\epsilon}}(u')^{\frac{1}{1+\epsilon}})$ and some constant $c \in (1,2)$. Putting together Equations ~\eqref{eq:proof_lb_u_1}, ~\eqref{eq:proof_lb_u_2} and ~\eqref{eq:proof_lb_u_3}, we have:
\begin{align*}
     & \max \left\{ \frac{R_T(\texttt{Alg},\bm{\nu})}{(uT)^{\frac{1}{1+\epsilon}}},  \frac{R_T(\texttt{Alg},\bm{\nu}')}{(u'T)^{\frac{1}{1+\epsilon}}}\right\}  \ge \\
     & \hspace{1.5cm} \ge \max \left\{\frac{\Delta \mathbb{E}_{\texttt{Alg},\bm{\nu}}[N_1(T)]}{(uT)^{\frac{1}{1+\epsilon}}},  \frac{\Delta}{8} (u')^{-\frac{1}{\epsilon+1}}T^{\frac{\epsilon}{\epsilon+1}}  \exp \left( - c \mathbb{E}_{\texttt{Alg},\bm{\nu}}[N_1(T)] (2\Delta)^{1+\frac{1}{\epsilon}} (u')^{-\frac{1}{\epsilon}} \right) \right\} \\
     & \hspace{1.5cm} \ge  \frac{\Delta}{2} \left( \frac{ \mathbb{E}_{\texttt{Alg},\bm{\nu}}[N_1(T)]}{(uT)^{\frac{1}{1+\epsilon}}} +  \frac{1}{8} (u')^{-\frac{1}{\epsilon+1}} T^{\frac{\epsilon}{\epsilon+1}}  \exp \left( - c \mathbb{E}_{\texttt{Alg},\bm{\nu}}[N_1(T)] (2\Delta)^{\frac{1+\epsilon}{\epsilon}}(u')^{-\frac{1}{\epsilon}} \right) \right) \\
     & \hspace{1.5cm} \ge  \frac{\Delta}{2} \min_{x \in [0,T]} \left\{ \frac{x}{(uT)^{\frac{1}{1+\epsilon}}} +  \frac{1}{8} (u')^{-\frac{1}{\epsilon+1}} T^{\frac{\epsilon}{\epsilon+1}}  \exp \left( - cx (2\Delta)^{\frac{1+\epsilon}{\epsilon}}(u')^{-\frac{1}{\epsilon}} \right) \right\} \eqqcolon g(x)
\end{align*}
The latter is a convex function of $x$ and the minimization can be carried out in closed form, vanishing the derivative and finding: 
\begin{align*}
    x^* = c^{-1} (2\Delta)^{-\frac{1+\epsilon}{\epsilon}} (u')^{\frac{1}{\epsilon}} \log \left( \frac{T u^{\frac{1}{\epsilon+1}}}{8(u')^{\frac{1}{\epsilon}+ \frac{1}{\epsilon+1}}} c (2\Delta)^{\frac{1+\epsilon}{\epsilon}}  \right),
\end{align*}
which leads to:
\begin{align*}
    g(x^*) 
    &= \frac{\Delta}{2} (uT)^{-\frac{1}{\epsilon+1}} c^{-1} (2\Delta)^{-\frac{1+\epsilon}{\epsilon}} (u')^{\frac{1}{\epsilon}}  \log \left( \frac{T u^{\frac{1}{\epsilon+1}}}{8(u')^{\frac{1}{\epsilon}+   \frac{1}{\epsilon+1}}} e c (2\Delta)^{\frac{1+\epsilon}{\epsilon}}  \right).
\end{align*}
We choose $\Delta$ such that:
\begin{align*}
    \frac{T u^{\frac{1}{\epsilon+1}}}{8(u')^{\frac{1}{\epsilon}+   \frac{1}{\epsilon+1}}} c (2\Delta)^{\frac{1+\epsilon}{\epsilon}}  = e^\epsilon,
\end{align*}
resulting in $\Delta = 2^{\frac{2\epsilon-1}{1+\epsilon}}e^{\frac{\epsilon^2}{1+\epsilon} } (cT)^{-\frac{\epsilon}{\epsilon+1}} u^{-\frac{\epsilon}{(\epsilon+1)^2}} (u')^{\frac{1+2\epsilon}{(\epsilon+1)^2}}$.
This implies, after some calculations, that:
\begin{align*}
    g(x^*) & = c^{-\frac{\epsilon}{\epsilon+1}}  2^{-\frac{2\epsilon+5}{\epsilon+1}} (1+\epsilon)e^{-\frac{\epsilon}{\epsilon+1}} u^{-\frac{\epsilon}{(\epsilon+1)^2} } (u')^{\frac{\epsilon}{(\epsilon+1)^2}} \ge c_1 \left( \frac{u'}{u} \right)^{\frac{\epsilon}{(\epsilon+1)^2}},
\end{align*}
where $c_1>0$ is a value independent of $T$ and both $u$ and $u'$.
Finally, we have that
\begin{align*} 
    \max \left\{ \frac{R_T(\texttt{Alg},\bm{\nu})}{(uT)^{\frac{1}{1+\epsilon}}},  \frac{R_T(\texttt{Alg},\bm{\nu}')}{(u'T)^{\frac{1}{1+\epsilon}}}\right\} \ge  c_1 \left( \frac{u'}{u} \right)^{\frac{\epsilon}{(\epsilon+1)^2}}.
\end{align*}
We observe that $\Delta <\left(\frac{1}{2}\right)^{\frac{2\epsilon+1}{1+\epsilon}}(u')^{\frac{1}{1+\epsilon}}$  for sufficiently large $T$. This concludes the proof of the second statement. For the first statement, we observe that, since $u'\ge u$ can be taken arbitrarily large, the right-hand side of this inequality can be arbitrarily large. 
\end{proof}

\epsAdaptiveLowerBound*
\begin{proof}
We start by constructing two heavy-tailed bandit instances with different maximum orders of moment $\epsilon$ and $\epsilon'$, where $0 < \epsilon' < \epsilon < 1$. For the sake of simplicity, but without loss of generality, we will assume a common (and known to the algorithm) maximum moment of $u=1$.
\paragraph{Base instance}
\begin{equation}
   \bm{\nu}= \begin{cases}
        &\nu_1 = \delta_0,\\
        &\nu_2 = (1+\Delta\gamma - \gamma^{1+\epsilon}) \delta_0 + (\gamma^{1+\epsilon}-\Delta \gamma) \delta_{1/\gamma}
    \end{cases},
\end{equation}
where $\Delta \in [0, 1/2]$ and $\gamma = (2\Delta)^{\frac{1}{\epsilon}}$.
Thus, we have $\mu_1 = 0$ and $\mu_2 = \Delta$. Furthermore, $\mathbb{E}_{X\sim\nu_1}[|X|^\alpha] = 0$ and $\mathbb{E}_{X\sim\nu_2}[|X|^\alpha]=2^{\frac{1-\alpha}{\epsilon}} \Delta^{\frac{1+\epsilon-\alpha}{\epsilon}}$,
which are guaranteed to be bounded by a constant smaller than $1$ only if $\alpha \le \epsilon+1$. Thus, this instance admits moments finite only up to order $\epsilon+1$, \emph{i.e.,} $\bm{\nu} \in \mathcal{P}(\epsilon,1)^2$. Moreover, the optimal arm is arm $2$.
\paragraph{Alternative instance}
\begin{equation}
   \bm{\nu}'= \begin{cases}
        &\nu_1' =  (1- (\gamma')^{1+\epsilon'}) \delta_0 + (\gamma')^{1+\epsilon'} \delta_{1/\gamma'}, \\
        &\nu_2' = \nu_2
    \end{cases},
\end{equation}
where $\Delta \in [0,1/2]$ and $\gamma' = (2\Delta)^{\frac{1}{\epsilon'}}$.
Thus, we have $\mu_1' = 2\Delta$ and $\mu_2' = \Delta$. Furthermore, $\mathbb{E}_{X\sim\nu_1'}[|x|^\alpha] = (2\Delta)^{\frac{1+\epsilon'-\alpha}{\epsilon'}}$ and $\mathbb{E}_{X\sim\nu_2'}[|x|^\alpha]=2^{\frac{1-\alpha}{\epsilon}} \Delta^{\frac{1+\epsilon-\alpha}{\epsilon}}$, which are guaranteed to be bounded by a constant smaller than $1$ only if $\alpha \le \epsilon'+1$. Thus, this instance admits moments finite only up to order $\epsilon'+1$, \emph{i.e.,} $\bm{\nu} \in \mathcal{P}(\epsilon',1)^2$. Moreover, the optimal arm is arm $1$. 

We will prove, that for any algorithm \texttt{Alg} it holds that:
\begin{align*}
     \max \left\{ \frac{R_T(\texttt{Alg},\bm{\nu})}{T^{\frac{1}{1+\epsilon}}},  \frac{R_T(\texttt{Alg},\bm{\nu}')}{T^{\frac{1}{1+\epsilon'}}}\right\} \ge f(T,\epsilon,\epsilon'),
\end{align*}
being $f$ a function increasing in $T$. The proof emulates the analyses and steps performed to prove Theorem \ref{thr:u_adaptive_lb}.
First, we observe that:
\begin{align}\label{eq:proof_lb_eps_1}
      \max \left\{ \frac{R_T(\texttt{Alg},\bm{\nu})}{T^{\frac{1}{1+\epsilon}}},  \frac{R_T(\texttt{Alg},\bm{\nu}')}{T^{\frac{1}{1+\epsilon'}}}\right\} \ge  \frac{R_T(\texttt{Alg},\bm{\nu})}{T^{\frac{1}{1+\epsilon}}} = \frac{\Delta \mathbb{E}_{\texttt{Alg},\bm{\nu}}[N_1(T)]}{T^{\frac{1}{1+\epsilon}}}, 
\end{align}
where $\mathbb{E}_{\texttt{Alg},\bm{\nu}}[N_1(T)]$ is the expected number of times arm $1$ is pulled over the horizon $T$. 

Second, recalling which are the optimal arms in the two instances and that $\epsilon' < \epsilon$, we have:
\begin{equation}
\begin{aligned}\label{eq:proof_lb_eps_2}
    & \max \left\{ \frac{R_T(\texttt{Alg},\bm{\nu})}{T^{\frac{1}{1+\epsilon}}},  \frac{R_T(\texttt{Alg},\bm{\nu}')}{T^{\frac{1}{1+\epsilon'}}}\right\}  \ge \\
    & \hspace{1.5cm} \ge T^{-\frac{1}{\epsilon'+1}} \max \left\{ \frac{\Delta T}{2} \mathbb{P}_{\texttt{Alg},\bm{\nu}}\left( N_1(T) \ge \dfrac{T}{2}\right) , \frac{\Delta T}{2} \mathbb{P}_{\texttt{Alg},\bm{\nu}'}\left( N_1(T) < \dfrac{T}{2}\right)  \right\} \\
    & \hspace{1.5cm} \ge \frac{\Delta}{4} T^{\frac{\epsilon'}{\epsilon'+1}} \left( \mathbb{P}_{\texttt{Alg},\bm{\nu}}\left( N_1(T) \ge \dfrac{T}{2}\right) +  \mathbb{P}_{\texttt{Alg},\bm{\nu}'}\left( N_1(T) < \dfrac{T}{2}\right) \right) \\
    & \hspace{1.5cm} \ge \frac{\Delta}{8} T^{\frac{\epsilon'}{\epsilon'+1}}  \exp \left( - \mathbb{E}_{\texttt{Alg},\bm{\nu}}[N_1(T)] D_{\text{KL}}(\nu_1 \| \nu_1') \right).
\end{aligned}
\end{equation}
where we used Bretagnolle-Huber inequality and divergence decomposition, together with  $\max\{a,b\} \ge \frac{1}{2}(a+b)$ for $a,b\ge 0$. Let us now compute the KL-divergence, noting that $\nu_1 
\ll \nu_1'$:
\begin{equation}
\begin{aligned}\label{eq:proof_lb_eps_3}
    D_{\text{KL}}(\nu_1 \| \nu_1') & = \nu_1(0) \log \frac{\nu_1(0)}{\nu_1'(0)} \\
    & = \log \frac{1}{1-(2\Delta)^{\frac{1+\epsilon'}{\epsilon'}}} \le c (2\Delta)^{\frac{1+\epsilon'}{\epsilon'}},
\end{aligned}
\end{equation}
for $\Delta \in [0,1/4]$ and some constant $c\in(1,2)$.
Putting together Equations~\eqref{eq:proof_lb_eps_1}, \eqref{eq:proof_lb_eps_2} and~\eqref{eq:proof_lb_eps_3}, we have:
\begin{align*}
     \max &\left\{ \frac{R_T(\texttt{Alg},\bm{\nu})}{T^{\frac{1}{1+\epsilon}}},  \frac{R_T(\texttt{Alg},\bm{\nu}')}{T^{\frac{1}{1+\epsilon'}}}\right\} \\ 
     & \ge \max \left\{\frac{\Delta \mathbb{E}_{\texttt{Alg},\bm{\nu}}[N_1(T)]}{T^{\frac{1}{1+\epsilon}}},  \frac{\Delta}{8} T^{\frac{\epsilon'}{\epsilon'+1}}  \exp \left( - c \mathbb{E}_{\texttt{Alg},\bm{\nu}}[N_1(T)] (2\Delta)^{\frac{1+\epsilon'}{\epsilon'}} \right) \right\} \\
     & \ge  \frac{\Delta}{2} \left( \frac{ \mathbb{E}[N_1(T)]}{T^{\frac{1}{1+\epsilon}}} +  \frac{1}{8} T^{\frac{\epsilon'}{\epsilon'+1}}  \exp \left( - c \mathbb{E}[N_1(T)] (2\Delta)^{\frac{1+\epsilon'}{\epsilon'}} \right) \right) \\
     & \ge \frac{\Delta}{2} \min_{x \in [0,T]} \left\{  \frac{ x}{T^{\frac{1}{1+\epsilon}}} +  \frac{1}{8} T^{\frac{\epsilon'}{\epsilon'+1}}  \exp \left( - c x (2\Delta)^{\frac{1+\epsilon'}{\epsilon'}} \right) \right\} \eqqcolon g(x).
\end{align*}
The latter is a convex function of $x$ and the minimization can be carried out in closed form vanishing the derivative and obtaining: 
\[
   x^* = c^{-1} (2\Delta)^{-\frac{1+\epsilon'}{\epsilon'}} \log \left( \frac{T^{\frac{1}{\epsilon+1} + \frac{\epsilon'}{1+\epsilon'}}}{8} c (2\Delta)^{\frac{1+\epsilon'}{\epsilon'}}  \right),
\]
which leads to:
\begin{align*}
    g(x^*) 
    & = \frac{\Delta}{2} T^{-\frac{1}{\epsilon+1}} c^{-1} (2\Delta)^{-\frac{1+\epsilon'}{\epsilon'}} \log \left( \frac{T^{\frac{1}{\epsilon+1} + \frac{\epsilon'}{1+\epsilon'}}}{8} e c (2\Delta)^{\frac{1+\epsilon'}{\epsilon'}}  \right).
\end{align*}
We take $\Delta$ such that:
\begin{align*}
    & \frac{T^{\frac{1}{\epsilon+1} + \frac{\epsilon'}{1+\epsilon'}}}{8}  c (2\Delta)^{\frac{1+\epsilon'}{\epsilon'}} = 1 ,
\end{align*}
resulting in $ \Delta = 2^{\frac{2\epsilon'-1}{1+\epsilon'}}c^{-\frac{\epsilon'}{1+\epsilon'}} T^{-\frac{\epsilon'}{1+\epsilon'} \left(\frac{1}{\epsilon+1} + \frac{\epsilon'}{1+\epsilon'} \right)}$.
This imply, after some calculations, that:
\begin{align*}
    g(x^*) 
    & =2^{\frac{-2\epsilon'-5}{1+\epsilon'}}c^{-\frac{\epsilon'}{1+\epsilon'}} T^{\frac{\epsilon'(\epsilon-\epsilon')}{(1+\epsilon')^2(1+\epsilon)}} \ge c_2 T^{\frac{\epsilon'(\epsilon-\epsilon')}{(1+\epsilon')^2(1+\epsilon)}}.
\end{align*}
where $c_2>0$ is a value independent of $T$ and can be always selected to be $\epsilon$ and $\epsilon'$.
Finally, we have that:
\begin{align*}
     \max \left\{ \frac{R_T(\texttt{Alg},\bm{\nu})}{T^{\frac{1}{1+\epsilon}}},  \frac{R_T(\texttt{Alg},\bm{\nu}')}{T^{\frac{1}{1+\epsilon'}}}\right\} \ge  c_2 T^{\frac{\epsilon'(\epsilon-\epsilon')}{(1+\epsilon')^2(1+\epsilon)}}.
\end{align*}
We observe that $\Delta < 1/4$ for sufficiently large $T$. We conclude by observing that the exponent of $T$ is maximized by taking $\epsilon=1$ and $\epsilon'=1/3$.
\end{proof}

\AssLowerBound*
\begin{proof}
We will construct instances using the following prototype of reward distribution, defined for $y \in (0, u^{\frac{1}{1+\epsilon}})$ and $\Delta \in (0, u^{\frac{1}{1+\epsilon}})$ : 
\begin{align}
\label{eq:instance_construction}
    \rho_y = \left( 1-y^{1+\frac{1}{\epsilon}} u^{-\frac{1}{\epsilon}} \right)\delta_0 +  \left(y^{1+\frac{1}{\epsilon}} u^{-\frac{1}{\epsilon}}\right)\delta_{-u^{\frac{1}{\epsilon}}\Delta^{-\frac{1}{\epsilon}}}. 
\end{align}
The two instances are constructed by the means of Equation \eqref{eq:instance_construction}. Note that we have:
\begin{align}
    &\E_{X \sim \rho_y}[X] = -y^{1+\frac{1}{\epsilon}}\Delta^{-\frac{1}{\epsilon}},\\
    & \E_{X \sim \rho_{y}}[|X|^{1+\epsilon}] = y^{1+\frac{1}{\epsilon}}\Delta^{-1-\frac{1}{\epsilon}}u \le u,
\end{align}
for every $0 \le y \le \Delta$.
\paragraph{Base instance}
\begin{align*}
\bm{\nu} = \begin{cases}
    & \nu_1 = \rho_{\left(\frac{2}{3}\right)^{\frac{\epsilon}{1+\epsilon}}\Delta}, \\
    & \nu_j = \rho_{\Delta}, \qquad\qquad j \in [K] \setminus \{1\}.
\end{cases}
\end{align*}
\paragraph{Alternative instance}
\begin{align*}
\bm{\nu}' = \begin{cases}
    & \nu_1' =  \rho_{\left(\frac{2}{3}\right)^{\frac{\epsilon}{1+\epsilon}}\Delta}, \\
    & \nu_i' = \rho_{\left(\frac{1}{3}\right)^{\frac{\epsilon}{1+\epsilon}}\Delta}, \\
    & \nu_j' = \rho_{\Delta}, \qquad\qquad j \in [K] \setminus \{1,i\},
\end{cases}
\end{align*}
where $i \in \text{argmin}_{j \neq 1} \mathbb{E}_{\texttt{Alg},\boldsymbol{\nu'}}[N_j(T)]$. For the base instance, we have $\mu_1 = -2\Delta/3$ and $\mu_j = -\Delta$ for all $j \neq 1$; whereas for the alternative instance $\mu_j'=\mu_j$ for all $j \neq i$ and $\mu_i'=-\Delta/3$. Both instances satisfy Assumption \ref{ass:truncated_np}, being the support a subset made of non-positive numbers. Moreover, for the base instance, the optimal arm is $1$ and for the alternative instance, the optimal arm is $i$.
Using the Bretagnolle-Huber inequality, we obtain:
\begin{align*}
    R_T (\texttt{Alg},\bm{\nu}) + R_T(\texttt{Alg},\bm{\nu}') 
    & \ge \frac{\Delta T}{6}\left(\mathbb{P}_{\texttt{Alg},\bm{\nu}}\left(N_1 \le \frac{T}{2}\right) + \mathbb{P}_{\texttt{Alg},\bm{\nu}'}\left(N_1 > \frac{T}{2}\right)\right) \\
    &\ge \frac{\Delta T}{6}\text{exp}\left(-\E_{\texttt{Alg},\bm{\nu}}[N_i(T)]D_{\text{KL}}({\nu}_i || {\nu}'_i)\right)
\end{align*}
We recall that by the definition of $i$, we have that $\E_{\texttt{Alg},\bm{\nu}}[N_i(T)] \le \frac{T}{K-1}$. We now compute the Kullback-Leibler divergence between the two instances:
\begin{align*}
    D_{\text{KL}}(\nu_i || \nu_i') 
    &= \Delta^{1+\frac{1}{\epsilon}}u^{-\frac{1}{\epsilon}}\log\left(\frac{\Delta^{1+\frac{1}{\epsilon}}u^{-\frac{1}{\epsilon}}}{\frac{1}{3}\Delta^{1+\frac{1}{\epsilon}}u^{-\frac{1}{\epsilon}}}\right) +  \underbrace{ (1-\Delta^{1+\frac{1}{\epsilon}}u^{-\frac{1}{\epsilon}})\log\left(\frac{1-\Delta^{1+\frac{1}{\epsilon}}u^{-\frac{1}{\epsilon}}}{1-\frac{1}{3}\Delta^{1+\frac{1}{\epsilon}}u^{-\frac{1}{\epsilon}}}\right)}_{\le 0} \\
    &\le \Delta^{1+\frac{1}{\epsilon}}u^{-\frac{1}{\epsilon}}\log 3.
\end{align*}
Plugging this result, we finally get:
\begin{align*}
     R_T (\texttt{Alg},\bm{\nu}) + R_T(\texttt{Alg},\bm{\nu}') \ge \frac{\Delta T}{6}\exp\left(-\frac{T}{K-1}\Delta^{1+\frac{1}{\epsilon}}u^{-\frac{1}{\epsilon}}\log 3\right).
\end{align*}
We conclude the proof by noting that $\max\{x,y\} > \frac{1}{2}(x+y)$ and setting $\Delta = \frac{1}{2}\left(\frac{K-1}{T}u^\frac{1}{\epsilon}\frac{1}{\log 3}\right)^{\frac{\epsilon}{1+\epsilon}}$. Finally, we have:
\begin{align*}
    \max\{R_T (\texttt{Alg},\bm{\nu}), R_T(\texttt{Alg},\bm{\nu}')\} \ge c_3K^\frac{\epsilon}{1+\epsilon}(uT)^{\frac{1}{1+\epsilon}},
\end{align*}
for some constant $c_3>0$ independent of $T$, $u$, $\epsilon$ and $K$. 
\end{proof}
\subsection{Estimator}
\label{apx:estimator}
\thrConcentrationMGeneric*
\begin{proof}
Since $M$ is computed independently of $\mathbf{X}$, the trimmed samples $X_i \mathbbm{1}_{\{|X_i| \le M\}}$ remain independent. Thus, with probability at least $1-\delta$, we have:
\begin{align*}
    \mu - \widehat{\mu}_s(\mathbf{X};M) & = \mathbb{E}[X_1] - \frac{1}{s} \sum_{t=1}^{s} X_t \mathbbm{1}_{|X_t| \leq M} \\
    & = \frac{1}{s} \sum_{t=1}^{s}\left(\mathbb{E} [X_1]-\mathbb{E}\left[X_t \mathbbm{1}_{|X_t| \leq M} \right] \right)+\frac{1}{s} \sum_{t=1}^{s}\left(\mathbb{E}\left[X_t \mathbbm{1}_{|X_t| \leq M} \right]- X_t \mathbbm{1}_{|X_t| \leq M} \right) \\
    & = \frac{1}{s} \sum_{t=1}^{s} \mathbb{E} [X_t \mathbbm{1}_{\left|X_{t}\right| > M}] + \frac{1}{s} \sum_{t=1}^{s}\left(\mathbb{E}\left[X_t \mathbbm{1}_{|X_t| \leq M} \right]- X_t \mathbbm{1}_{|X_t| \leq M} \right)  \\
    & \stackrel{(*)}{\leq} \frac{1}{s} \sum_{t=1}^{s}\left(\mathbb{E}\left[X_t \mathbbm{1}_{|X_t| \leq M} \right]- X_t \mathbbm{1}_{|X_t| \leq M} \right) \\
    & \stackrel{(**)}{\leq} \sqrt{\frac{2 V_s(\mathbf{Y}) \log 2\delta^{-1}}{s}}+\frac{ 14 M \log 2\delta^{-1}}{3(s-1)} \\
    & \le \sqrt{\frac{2 V_s(\mathbf{Y}) \log 2\delta^{-1}}{s}}+\frac{10 M \log 2\delta^{-1}}{s}
\end{align*}
Note that in step (\raisebox{-0.5ex}{*}) we used Assumption \ref{ass:truncated_np} to make the first term vanish. In step (\raisebox{-0.5ex}{**}), instead, we used \emph{empirical Bernstein inequality} \citep{maurer2009empirical} recalling that the trimmed random variables range in $[-M,M]$. We also use $\frac{1}{s-1}\le \frac{2}{s}$ in the last step for $s \ge 2$.
\end{proof}

\begin{restatable}[Uniqueness of Solution of Equation \eqref{eq:M_hat},  \citet{wang2021new}]{prop}{Proposition}\label{thr:uniqueness}
Let $\mathbf{X} = \{X_1, \ldots, X_s\}$ be a set of real numbers. If:
    \begin{equation} \label{eq:uniqueness_M}
        0 < c\log\delta^{-1} <\sum_{j \in [s]} \mathbbm{1}_{\{X_i\neq 0\}},
    \end{equation}
then Equation \eqref{eq:M_hat} admits a unique positive solution.
\end{restatable}

\label{proof:M_properties}




\boundsonM*
\begin{proof}
The proof makes use of the concentration inequality for self-bounding random variables \citep{maurer2006concentration, maurer2009empirical}. Let $M > 0$, for every $i \in \dsb{s}$, we define the random variable:
\begin{align*}
    U_{i,M} \coloneqq \min\left\{\left(\frac{X_i}{M}\right)^2, 1\right\},
\end{align*}
that ranges in $[0,1]$. Furthermore, let: $Z_M(\mathbf{X}) \coloneqq \sum_{i=1}^s U_{i,M}$, ranging in $[0,s]$. Let us denote $\overline{U}_{M}(\mathbf{X}) \coloneqq Z_M(\mathbf{X}) /s$, we observe that, given these definitions, the equation we want to solve for non-zero roots becomes:
\begin{align}
    \overline{U}_{M}(\mathbf{X})  - \frac{c \log \delta^{-1}}{s} = 0.
\end{align}
We start by showing that $Z_M(\mathbf{X})$ satisfies the assumptions of Theorem 13 of \cite{maurer2006concentration}, in particular, let $a\ge 1$, we have:
\begin{align}
        Z_M(\mathbf{X}) - \inf_{y\in\mathbb{R}} Z_M(\mathbf{X}_{y,k}) \le 1, \quad \forall k \in[s], \label{eq:ass1}\\
        \sum_{k=1}^{s} \left( Z_M(\mathbf{X}) - \inf_{y\in\mathbb{R}} Z_M(\mathbf{X}_{y,k}) \right)^2 \le aZ_M(\mathbf{X}), \label{eq:ass2}
\end{align}
where $\mathbf{X}_{y,k}$ is obtained by replacing with $y$ the $k$-th element $X_k$ of the set $\mathbf{X}$. Indeed, Equation \eqref{eq:ass1} follows as:
\begin{align*}
    Z_M(\mathbf{X}) - \inf_{y\in\mathbb{R}} Z_M(\mathbf{X}_{y,k}) &= U_{k,M} -  \inf_{y\in\mathbb{R}} \min\left\{\left(\frac{y}{M}\right)^2, 1\right\} = U_{k,M} \le 1, \quad \forall k \in [s].
\end{align*}
Similarly, we set $a=1$ and obtain Equation \eqref{eq:ass2} as follows:
\begin{align*}
   \sum_{k=1}^{s} \left( Z_M(\mathbf{X}) - \inf_{y\in\mathbb{R}} Z_M(\mathbf{X}_{y,k}) \right)^2 &= \sum_{k=1}^{s} \left(U_{k,M} -  \inf_{y\in\mathbb{R}} \min\left\{\left(\frac{y}{M}\right)^2, 1\right\}\right)^2\\
   &\le \sum_{k=1}^{s} U_{k,M}^2 \\
   &\le \sum_{k=1}^{s} U_{k,M} \\
   &= Z_M(\mathbf{X}),
\end{align*}
since $U_{k,M} \le 1$.
Using Theorem 13 from \cite{maurer2006concentration} with $a=1$, for the right tail of the distribution, we have for every $\epsilon>0$:
\begin{align*}
    \mathbb{P}\left(\mathbb{E}[Z_M(\mathbf{X})] - Z_M(\mathbf{X}) > s\epsilon \right) \le \exp\left(\dfrac{-\epsilon^2s^2}{2\mathbb{E}[Z_M(\mathbf{X})]}\right)
\end{align*}
By letting $\epsilon =
\sqrt{\dfrac{2\mathbb{E}[\overline{U}_{M}(\mathbf{X})]\log 2\delta^{-1}}{s}}$ and recalling the definition of $\overline{U}_{M}(\mathbf{X})$, we obtain:
\begin{align*}
    \mathbb{P}\left(\mathbb{E}[\overline{U}_{M}(\mathbf{X})] - \overline{U}_{M}(\mathbf{X}) > \sqrt{\dfrac{2\mathbb{E}[\overline{U}_{M}(\mathbf{X})]\log \delta^{-1}}{s}} \right) \le \delta,
\end{align*}
which implies, after some algebraic manipulations (see Theorem 10 of \citep{maurer2009empirical}), the following:
\begin{align*}
    \mathbb{P}\left(\sqrt{\mathbb{E}[\overline{U}_{M}(\mathbf{X})]} - \sqrt{\overline{U}_{M}(\mathbf{X})} > \sqrt{\frac{2 \log \delta^{-1}}{s}}\right) \le {\delta}.
\end{align*}
A similar inequality holds for the left tail:
\begin{align*}
    \mathbb{P}\left(Z_M(\mathbf{X})-\mathbb{E}[Z_M(\mathbf{X})]  > s\epsilon \right) \le \exp\left(\dfrac{-\epsilon^2s^2}{2\mathbb{E}[Z_M(\mathbf{X})]+\epsilon s}\right),
\end{align*}
with similar steps, we obtain:
\begin{align*}
    \mathbb{P}\left(\sqrt{\overline{U}_{M}(\mathbf{X})} - \sqrt{\mathbb{E}[\overline{U}_{M}(\mathbf{X})]} > \sqrt{\frac{2\log \delta^{-1}}{s}} \right) \le {\delta}.
\end{align*}
With a union bound over the two inequalities on the left and the right tail, we finally get:
\begin{align}
\label{eq:conc_ineq_M}
    \mathbb{P}\left(\left|\sqrt{\overline{U}_{M}(\mathbf{X})} - \sqrt{\mathbb{E}[\overline{U}_{M}(\mathbf{X})]}\right| > \sqrt{\frac{2\log \delta^{-1}}{s}} \right) \le {2\delta}.
\end{align}
Let us now define $\widehat{M}_s(\delta)$ random variable corresponding to the solution of the equation:
\begin{align*}
    \overline{U}_{\widehat{M}_s(\delta)}(\mathbf{X}) = \frac{c\log \delta^{-1}}{s},
\end{align*}
where $c>0$. To control the bounds on $\widehat{M}$, we define the following auxiliary (non-random) quantities:
\begin{align}
     \sqrt{\overline{U}_{M}^+} \coloneqq \sqrt{\mathbb{E}[\overline{U}_{M}(\mathbf{X})]} + \sqrt{\frac{2\log \delta^{-1}}{s}} \quad \text{and} \quad  \sqrt{\overline{U}_{M}^-} \coloneqq \sqrt{\mathbb{E}[\overline{U}_{M}(\mathbf{X})]} - \sqrt{\frac{2\log \delta^{-1}}{s}}. 
\end{align}
Thanks to Equation~\ref{eq:conc_ineq_M}, we have, for every $M \ge 0$,  that $\mathbb{P}(\overline{U}_{M}^- \le \overline{U}_{M}(\mathbf{X}) \le \overline{U}_{M}^+) \ge 1-2\delta$. Furthermore, let $M^+(\delta),M^-(\delta)>0$, the solutions of the following (non-random) equations:
\begin{align}
    U_{M^+(\delta)}^+ = \frac{c\log \delta^{-1}}{s} \qquad \text{and} \qquad  U_{M^-(\delta)}^- = \frac{c\log \delta^{-1}}{s}.
\end{align}
Since $\mathbb{P}(\overline{U}_{M}^- \le \overline{U}_{M}(\mathbf{X}) \le \overline{U}_{M}^+) \ge 1-2\delta$, it follows that $\mathbb{P}(M^-(\delta) \le \widehat{M}_s(\delta) \le M^+(\delta))\ge 1-2\delta$. We now proceed at lower bounding $M^-(\delta)$ and upper bounding $M^+(\delta)$:
\begin{align}
     \sqrt{\frac{c \log \delta^{-1}}{s}} &= \sqrt{U_{M^-(\delta)}^-}   \\
     & = \sqrt{\mathbb{E}[\overline{U}_{M^-(\delta)}(\mathbf{X})]} - \sqrt{\frac{2\log \delta^{-1}}{s}}\\
     & \ge \sqrt{\mathbb{P}\left( |X_1| \ge M^-(\delta) \right)} - \sqrt{\frac{2\log \delta^{-1}}{s}} \\
     & \ge \sqrt{\mathbb{P}\left( |X_1| \ge \widehat{M}_s(\delta) \right)} - \sqrt{\frac{2\log \delta^{-1}}{s}},
\end{align}
where the last but one inequality follows from:
\begin{align}
    \mathbb{E}[\overline{U}_{M}(\mathbf{X})] = \mathbb{E}\left[ \min \left\{\left(\frac{X_1}{M}\right)^2,1 \right\}\right] \ge \mathbb{P}\left(\left(\frac{X_1}{M}\right)^2 \ge 1 \right) = \mathbb{P}\left(|X_1| \ge M \right),
\end{align}
and the last inequality holds with probability $1-\delta$ and follows from the fact that $\widehat{M}_s(\delta) \ge  M^-(\delta)$. Similarly, we have:
\begin{align}
    \sqrt{\frac{c \log \delta^{-1}}{s}} &= \sqrt{U_{M^+(\delta)}^+}   \\
     & = \sqrt{\mathbb{E}[\overline{U}_{M^-(\delta)}(\mathbf{X})]} + \sqrt{\frac{2\log \delta^{-1}}{s}}\\
     & \le \sqrt{\frac{u}{(M^+(\delta))^{1+\epsilon}}} + \sqrt{\frac{2\log \delta^{-1}}{s}} \\
     & \le \sqrt{\frac{u}{(\widehat{M}_s(\delta))^{1+\epsilon}}} + \sqrt{\frac{2\log \delta^{-1}}{s}},
\end{align}
where the last but one inequality follows from:
\begin{align}
    \mathbb{E}[\overline{U}_{M}(\mathbf{X})] = \mathbb{E}\left[\min \left\{\left(\frac{X_1}{M}\right)^2,1 \right\}\right] \le M^{-1-\epsilon}\mathbb{E}\left[|X_1|^{1+\epsilon}\right]\le M^{-1-\epsilon} u,
\end{align}
and the last inequality holds with probability $1-\delta$ and follows from the fact that $\widehat{M}_s(\delta) \le  M^+(\delta)$. Thus, with probability $1-2\delta$, we have for $c > 2$:
\begin{align}
    \mathbb{P}\left( |X_1| > \widehat{M}_s(\delta) \right)\le (\sqrt{c}+\sqrt{2})^2 {\frac{\log \delta^{-1}}{s}} \quad \text{and} \quad \widehat{M}_s(\delta) \le \left( \frac{us}{(\sqrt{c}-\sqrt{2})^2\log \delta^{-1}} \right)^{\frac{1}{1+\epsilon}}.
\end{align}
\end{proof}

\TheoremConc*
\begin{proof}
The result is obtained by combining an application of Bernstein's inequality and the bounds on the threshold $\widehat{M}_s(\delta)$ of Lemma~\ref{thr:M_hat_bound}. Furthermore since $\widehat{M}_s(\delta) $ is independent of $\mathbf{X}$, we can condition on the value of $\widehat{M}_s(\delta)$. With probability $1-\delta$, we have:
\begin{align*}
    \widehat{\mu}_s&(\mathbf{X}; \widehat{M}_s(\delta)) - \mu  = \frac{1}{s} \sum_{i=1}^{s} X_i \mathbbm{1}_{|X_i| \leq \widehat{M}_s(\delta)} - \mathbb{E}[X_1] \\
    & = \frac{1}{s} \sum_{i=1}^{s}\left(X_i \mathbbm{1}_{|X_i| \leq \widehat{M}_s(\delta)}- \mathbb{E}\left[X_i \mathbbm{1}_{|X_i| \leq \widehat{M}_s(\delta)} \right] \right) - \frac{1}{s} \sum_{i=1}^{s}\left(\mathbb{E} [X_1]-\mathbb{E}\left[X_t \mathbbm{1}_{|X_i| \leq \widehat{M}_s(\delta)} \right] \right) \\
    & = \frac{1}{s} \sum_{i=1}^{s}\left(X_i \mathbbm{1}_{|X_i| \leq \widehat{M}_s(\delta)}- \mathbb{E}\left[X_i \mathbbm{1}_{|X_i| \leq \widehat{M}_s(\delta)} \right] \right) -\frac{1}{s} \sum_{i=1}^{s} \mathbb{E} [X_i\mathbbm{1}_{\left|X_{i}\right| > \widehat{M}_s(\delta)}] \\
    & \le \frac{1}{s} \sum_{i=1}^{s}\left(X_i \mathbbm{1}_{|X_i| \leq \widehat{M}_s(\delta)}- \mathbb{E}\left[X_i \mathbbm{1}_{|X_i| \leq \widehat{M}_s(\delta)} \right] \right) + \frac{1}{s} \sum_{i=1}^{s} \mathbb{E} [|X_i|\mathbbm{1}_{\left|X_{i}\right| > \widehat{M}_s(\delta)}] \\
    & \stackrel{(*)}\le \frac{1}{s} \sum_{i=1}^{s}\left(X_i \mathbbm{1}_{|X_i| \leq \widehat{M}_s(\delta)}- \mathbb{E}\left[X_i \mathbbm{1}_{|X_i| \leq \widehat{M}_s(\delta)} \right] \right) + \\
    & \quad \quad +\frac{1}{s}\sum_{i=1}^{s}\left(\mathbb{E}\left[|X_i|^{1+\epsilon}\right]^{\frac{1}{1+\epsilon}}\right)\left(\mathbb{E}\left[(\mathbbm{1}_{|X_i|>\widehat{M}_s(\delta)})^\frac{1+\epsilon}{\epsilon}\right]^\frac{\epsilon}{1+\epsilon}\right) \\
    & \stackrel{(**)}{\le} \sqrt{\frac{2 \widehat{M}_s(\delta)^{1-\epsilon} u \log \left(\delta^{-1}\right)}{s}}+\frac{\widehat{M}_s(\delta) \log \left(\delta^{-1}\right)}{3 s}  + \frac{1}{s} \sum_{i=1}^{s} \left(u^{\frac{1}{1+\varepsilon}}\right) \left( \mathbb{E}\left[ \mathbbm{1}_{\left|X_{i}\right| >\widehat{M}_s(\delta)}\right]^{\frac{\epsilon}{1+\varepsilon}}\right) \\
    & \leq \sqrt{\frac{2 \widehat{M}_s(\delta)^{1-\epsilon} u \log \left(\delta^{-1}\right)}{s}}+\frac{\widehat{M}_s(\delta) \log \left(\delta^{-1}\right)}{3 s} + u^{\frac{1}{1+\varepsilon}}  \mathbb{P}\left(|X_{i}| > \widehat{M}_s(\delta)\right)^{\frac{\epsilon}{1+\varepsilon}},
\end{align*}
where step (\raisebox{-0.5ex}{*}) follows from H\"older inequality, while step  (\raisebox{-0.5ex}{**}) is a consequence of Bernstein's inequality for bounded random variables.
To proceed further, we use Lemma \ref{thr:M_hat_bound} in union bound with the previously applied inequality. Thus, with probability at least $1-3\delta$, we have:
\begin{align*}
    \widehat{\mu}_s& (\mathbf{X};  \widehat{M}_s(\delta)) - \mu \le \\
    &  {\leq} \sqrt{{\frac{2 \left(\frac{us}{(\sqrt{c}-\sqrt{2})^2\log \delta^{-1}}\right)^{\frac{1-\epsilon}{1+\epsilon}} u \log \left(\delta^{-1}\right)}{s}}} + \frac{\left(\frac{us}{(\sqrt{c}-\sqrt{2})^2\log \delta^{-1}}\right)^{\frac{1}{1+\epsilon}} \log \left(\delta^{-1}\right)}{3 s}  \\
    & \hspace{5.35cm} + u^{\frac{1}{1+\varepsilon}} \left( (\sqrt{c}+\sqrt{2})^2 \frac{\log \delta^{-1}}{s} \right)^{\frac{\epsilon}{1+\epsilon}} \\
    & \leq \left( \frac{\sqrt{2}}{(\sqrt{c}-\sqrt{2})^{\frac{1-\epsilon}{1+\epsilon}}}+\frac{1}{3(\sqrt{c}-\sqrt{2})^{\frac{2}{1+\epsilon}}} + (\sqrt {c}+\sqrt{2})^\frac{2\epsilon}{1+\epsilon}\right)u^{\frac{1}{1+\epsilon}} \left(\frac{\log \delta^{-1}}{n} \right)^{\frac{\epsilon}{1+\epsilon}}\\
    & \leq 5.6 u^{\frac{1}{1+\epsilon}} \left(\frac{\log \delta^{-1}}{s} \right)^{\frac{\epsilon}{1+\epsilon}},
\end{align*}
where in the last passage we set $c=(1+\sqrt{2})^2$ and bounded the resulting expression for $\epsilon \in (0,1]$. A symmetric derivation leads to the second inequality. A union bound combined with renaming $s \leftarrow s/2$ and using $5.6\sqrt{2} \le 8$, concludes the proof.
\end{proof}
\TheoremID*
\begin{proof}
For notational convenience, in this derivation, we will perform the substitution $T \leftarrow \lfloor T/2 \rfloor$ and $t \leftarrow \tau$. For every arm $i\in[K]$ and round $t \in [T]$, let us define the event:
\begin{align}
    \mathcal{E}_{i,t} \coloneqq \left\{ \sum_{X \in \mathbf{X}_{i}'(t-1)} \mathbbm{1}_{\{X\neq 0\}} - 4\log t^3 > 0 \right\}.
\end{align}
Under event  $\mathcal{E}_{i,t}$ we do not incur in the forced exploration (FE) in line~\ref{line:cond} ensuring that every arm has collected at least $4\log t^3$ nonzero samples in $\mathbf{X}_i'$. Thus, we can decompose the expected number of pulls as follows:
\begin{align}
    \E[N_i^{\text{ALL}}(T)] & = \E\left[ \sum_{t\in[T]} \mathbbm{1}_{\{I_t=i \text{ and } \mathcal{E}_{i,t} \}} \right] + \E\left[ \sum_{t\in[T]} \mathbbm{1}_{\{I_t=i \text{ and } \mathcal{E}_{i,t}^\complement\}} \right] \\
    & = \E[N_i(T)] + \mathbb{E}[N_i^{\text{FE}}(T)].
\end{align}

\paragraph{Part I: Bounding the expected number of pulls for forced exploration.}
We first bound the expected number of pulls $\E[N_i^{\text{FE}}(T)]$ due to the forced exploration. Considering only the samples collected due to forced exploration, thanks to independence among these samples, we can see the required number of pulls as a sum of geometric random variables. Thus, we can compute an upper bound on the expectation as:
\begin{equation}
\label{eq:rr-pulls}
    \mathbb{E}_{\nu_i}[N_i^{\text{FE}}(T)] \le  \frac{4 \log T^3}{\mathbb{P}_{\nu_i}(|X|>0)}.
\end{equation}

\paragraph{Part II: Bounding the expected number of pulls for optimistic exploration.}
We define for every arm $i\in[K]$ and every round $t\in[T]$, the upper confidence bound as:
\begin{equation*}
    B_{i}(t) = \widehat{\mu}_{i}(t) + \sqrt{\frac{2V_{i}(t) \log t^3}{N_{i}(t-1)}}+  \frac{10 \widehat{M}_{i}(t) \log t^3}{N_{i}(t-1)},
\end{equation*}
where $N_i(t-1)$ is the number of times arm $i$ has been pulled up to time $t-1$, i.e., $N_i(t-1) = |\mathbf{X}_i(t-1)|$. 
We now show that if $I_{t}=i$, for an arm $i$ such that $\Delta_{i}>0$, then, one of the following four inequalities is true: 
\begin{align}
    & \text{either} \ B_{1}(t) \leq \mu_1, \label{eq:robust1} \\
    & \quad \text{or} \quad \widehat{\mu}_{i}(t)>\mu_{i}+ 5.6 u^{\frac{1}{1+\epsilon}} \left(\frac{\log t^3}{N_{i}(t-1)} \right)^{\frac{\epsilon}{1+\epsilon}}, \label{eq:robust2} \\
    & \quad \text{or} \quad N_{i}(t-1) < 20 \left(\dfrac{ u}{\Delta_{i}^{1+\epsilon}}\right)^\frac{1}{\epsilon}\log{t^3} \label{eq:robust3}, \\
    & \quad \text{or} \ \sqrt{V_{i}(t)} > \sqrt{\mathbb{E} [V_{i}(t)]}+2\widehat{M}_{i}(t)\sqrt{\frac{ \log t^3}{N_i(t-1)}} \label{eq:robust4}, \\
    & \quad \text{or} \ \widehat{M}_{i}(t) \geq \left(\dfrac{u N_i(t-1)}{\log t^3}\right)^{\frac{1}{1+\epsilon}} .\label{eq:robust5}
\end{align}
Indeed, assume that all five inequalities are false. Then we have
\begin{align*}
    B_{1}(t) & \stackrel{(\ref{eq:robust1})}{>} \mu_1 = \mu_{i}+\Delta_{i} \\
    & \stackrel{(\ref{eq:robust2})}{\geq} \widehat{\mu}_{i}(t) - 5.6u^{\frac{1}{1+\epsilon}} \left(\frac{\log t^3}{N_{i}(t-1)} \right)^{\frac{\epsilon}{1+\epsilon}} +\Delta_{i} \\
    & \stackrel{(*)}{\geq} \ \widehat{\mu}_{i}(t)+ \sqrt{\frac{2V_{i}(t) \log t^3}{N_{i}(t-1)}} +  \frac{ 10\widehat{M}_{i}(t) \log t^3}{N_{i}(t-1)} \\
    & = B_{i}(t).
\end{align*}
 The step marked with (\raisebox{-0.5ex}{*}) is a consequence of the fact that both \eqref{eq:robust3}, \eqref{eq:robust4} and \eqref{eq:robust5} are false.
In particular, we need to show that
\begin{align}
\label{eq:delta_bound}
    \Delta_{i} \geq 5.6 u^{\frac{1}{1+\epsilon}} \left(\frac{\log t^3}{N_{i}(t-1)} \right)^{\frac{\epsilon}{1+\epsilon}} + \sqrt{\frac{2V_{ i}(t) \log t^3}{N_{i}(t-1)}}+  \frac{ 10 \widehat{M}_{i}(t) \log t^3}{N_{i}(t-1)}. \tag{\raisebox{-0.5ex}{*}}
\end{align}
To do so, we make use of the following inequality derived by exploiting the independence between $\mathbf{X}_i(t-1)$ and $\mathbf{X}_i'(t-1)$:
\begin{align}
\label{eq:sigma_bound}
    \mathbb{E} [V_{i}(t)] & \leq \mathbb{E}\left[X^{2} \mathbbm{1}_{|X| \leq \widehat{M}_i(t)}\right]  \le \mathbb{E}\left[|X|^{1+\epsilon} \right]  \widehat{M}_i(t)^{1-\varepsilon}  \leq u\widehat{M}_i(t)^{1-\epsilon}.
\end{align}
Now, we make use of the fact that \eqref{eq:robust3}, \eqref{eq:robust4}, and \eqref{eq:robust5} are false together with \eqref{eq:sigma_bound}:
\begin{align*}
    \Delta_{i} & \stackrel{\eqref{eq:robust3}}\ge 20 u^{\frac{1}{1+\epsilon}} \left(\frac{\log t^3}{N_{i}(t-1)} \right)^{\frac{\epsilon}{1+\epsilon}} \\
    & \ge ({5.6}+\sqrt{2}+ 10 + 2\sqrt{2}) u^{\frac{1}{1+\epsilon}} \left(\frac{\log t^3}{N_{i}(t-1)} \right)^{\frac{\epsilon}{1+\epsilon}} \\
    & = 5.6 u^{\frac{1}{1+\epsilon}} \left(\frac{\log t^3}{N_{i}(t-1)} \right)^{\frac{\epsilon}{1+\epsilon}} + \sqrt{\dfrac{2\log t^3 u \left(\frac{u N_{i}(t-1)}{\log t^3}\right)^{\frac{1-\epsilon}{1+\epsilon}} }{N_{i}(t-1)}} +  \dfrac{(10 + 2\sqrt{2})\left(\frac{u N_{i}(t-1)}{\log t^3}\right)^{\frac{1}{1+\epsilon}}\log t^3}{N_{i}(t-1)} \\
    & \stackrel{\eqref{eq:robust5}}\ge 5.6 u^{\frac{1}{1+\epsilon}} \left(\frac{\log t^3}{N_{i}(t-1)} \right)^{\frac{\epsilon}{1+\epsilon}} + \sqrt{\dfrac{2 \log t^3 u \widehat{M}_{i}(t)^{1-\epsilon}}{N_{i}(t-1)} } + \dfrac{(10 + 2\sqrt{2})\widehat{M}_{i}(t) \log t^3}{N_{i}(t-1)} \\
    & \stackrel{\eqref{eq:sigma_bound}} \ge 5.6u^{\frac{1}{1+\epsilon}} \left(\frac{\log t^3}{N_{i}(t-1)} \right)^{\frac{\epsilon}{1+\epsilon}} + \sqrt{\frac{2\mathbb{E} [V_{i}(t)] \log t^3}{N_{i}(t-1)}} +  \frac{(10 + 2\sqrt{2}) \widehat{M}_{i}(t) \log t^3}{N_{i}(t-1)} \\
    & \stackrel{\eqref{eq:robust4}} \ge 5.6 u^{\frac{1}{1+\epsilon}} \left(\frac{\log t^3}{N_{i}(t-1)} \right)^{\frac{\epsilon}{1+\epsilon}} + \sqrt{\frac{2\log t^3}{N_{i}(t-1)}}\left[ \sqrt{V_{i}(t)} -2\widehat{M}_{i}(t)\sqrt{\frac{ \log t^3}{N_{i}(t-1)}} \right] \\
    & \qquad +\frac{ (10 + 2\sqrt{2}) \widehat{M}_{i}(t) \log t^3}{N_{i}(t-1)} \\
    &  \ge 5.6 u^{\frac{1}{1+\epsilon}} \left[\frac{\log t^3}{N_{i}(t-1)} \right]^{\frac{\epsilon}{1+\epsilon}} + 2\sqrt{\frac{V_i(t) \log t^3}{N_{i}(t-1)}}+  \frac{10 \widehat{M}_{i}(t)\log t^3}{N_i(t-1)}. \tag{\raisebox{-0.5ex}{*}}
\end{align*}
Finally, as a consequence of (\raisebox{-0.5ex}{*}), we have $
    B_{1}(t) > B_{i}(t)    $
but this is a contradiction since $T_t = i$. Thus, statements \eqref{eq:robust1} to \eqref{eq:robust5} cannot be false simultaneously. We now proceed with a union bound over all the possible values of $N_i(t-1)$ and of the previously introduced concentration inequalities to bound with $\frac{1}{t^3}$ the probabilities of events \eqref{eq:robust1}, \eqref{eq:robust2}, \eqref{eq:robust4}, and \eqref{eq:robust5} to be true:
\begin{align*}
    &\mathbb{P}\left(\exists N_{i}(t-1) \in [t] \,:\,\{\eqref{eq:robust1} \text{ is true} \}  \text{ or } \{\eqref{eq:robust2} \text{ is true}\}  \text{ or }  \{\eqref{eq:robust4} \text{ is true}\} \text{ or } \{\eqref{eq:robust5} \text{ is true}\}\right) \le \\
    & \qquad \le 6 \sum_{s=1}^t \frac{1}{t^3} = \frac{6}{t^2},
\end{align*}
where for \eqref{eq:robust4}, we used the second inequality of Theorem 10 of \citep{maurer2009empirical} (bounding $1/(n-1) \le 2/n$) and for \eqref{eq:robust5}, we used Theorem~\ref{thr:M_hat_bound}. To proceed, we introduce the quantity:
\begin{align*}
    v \coloneqq \left\lceil 60 \left(\dfrac{ u}{\Delta_{i}^{1+\varepsilon}} \right)^{{\frac{1}{\varepsilon}}}\log{T} \right\rceil.
\end{align*}
It's now time to bound the expected number of times each arm is pulled:
\begin{equation}
\begin{aligned}
\label{eq:n_pulls_bound}
\mathbb{E} [N_{i}(T)] =\mathbb{E} \left[\sum_{t=1}^{T} \mathbbm{1}_{\{I_{t}=i \text{and} \mathcal{E}_{i,t} \}}\right] & \leq v +\mathbb{E} \left[\sum_{t=v+1}^{T} \mathbbm{1}_{\{I_{t}=i \text { and } \{\eqref{eq:robust3} \text { is false }\}\}}\right] \\
& \le v +\mathbb{E} \left[\sum_{t=v+1}^{T}  \mathbbm{1}_{\{I_{t}=i \text { and }\{\eqref{eq:robust1} \text { or }\eqref{eq:robust2} \text{ or } \eqref{eq:robust4} \text{ or } \eqref{eq:robust5} \text { is true}\}\}}\right]\\
& \leq v +\sum_{t=v+1}^{T} \frac{6}{t^{2}}\\
& \leq v + 10.
\end{aligned}
\end{equation}
We now conclude the proof using the regret decomposition, considering the forced exploration through Equation \eqref{eq:rr-pulls} and that the effective number of pulls is doubled:
\begin{align*}
    R_T (\text{\algnameshort}, \boldsymbol{\nu})\le \sum_{i:\Delta_i>0} \left[\left(120\left(\frac{u}{\Delta_i}\right)^\frac{1}{\epsilon}+\frac{24\Delta_i}{\mathbb{P}_{\nu_i}(X\neq 0)}\right)\log \frac{T}{2} + 20\Delta_i \right].
\end{align*}

\end{proof}

\TheoremWC*
\begin{proof}
Let us fix $\Delta>0$, to be chosen later. We have:
\begin{align*}
    R_{T}&(\text{\algnameshort}, \boldsymbol{\nu}) = \sum_{i \in [K]} \Delta_i \left(2\mathbb{E}[N_i(T/2)] + \mathbb{E}_{\nu_i}[N_i^{\text{FE}}(T/2)]\right) \notag \\
    &= \sum_{i: \Delta_i \le \Delta} 2\Delta_i \mathbb{E}[N_i(T/2)] + \sum_{i: \Delta_i > \Delta} 2\Delta_i \mathbb{E}[N_i(T/2)] + \sum_{i: \Delta_i >0} \frac{27\Delta_i}{\mathbb{P}_{\nu_i}(X\neq 0)}\log \frac{T}{2} \notag \\
    &\le \Delta T + \sum_{i: \Delta_i > \Delta} 2\Delta_i \left(60\left(\frac{u}{\Delta_i^{{1+\epsilon}}}\right)^{{\frac{1}{\epsilon}}}\log \frac{T}{2} + 10\right)+ \sum_{i: \Delta_i >0} \frac{24\Delta_i}{\mathbb{P}_{\nu_i}(X\neq 0)}\log \frac{T}{2} \notag \\
    &\le \Delta T + 2 K \left(60\left(\frac{u}{\Delta}\right)^{{\frac{1}{\epsilon}}}\log \frac{T}{2} \right) + \sum_{i: \Delta_i >0} \left( \frac{24\Delta_i}{\mathbb{P}_{\nu_i}(X\neq 0)}\log \frac{T}{2} + 20\Delta_i\right) \notag \\
    &\stackrel{(*)}{\le} 120^{\frac{\epsilon}{1+\epsilon}} (1+\epsilon) \epsilon^{-\frac{\epsilon}{1+\epsilon}}\left( K\log \frac{T}{2}\right)^{\frac{\epsilon}{1+\epsilon}} (uT)^{\frac{1}{1+\epsilon}} + \sum_{i: \Delta_i >0} \left( \frac{24\Delta_i}{\mathbb{P}_{\nu_i}(X\neq 0)}\log \frac{T}{2} + 20\Delta_i\right)\\
    &  \stackrel{(**)}{\le}  46 \left( K\log \frac{T}{2}\right)^{\frac{\epsilon}{1+\epsilon}} (uT)^{\frac{1}{1+\epsilon}}+ \sum_{i: \Delta_i >0} \left( \frac{24\Delta_i}{\mathbb{P}_{\nu_i}(X\neq 0)}\log \frac{T}{2} + 20\Delta_i\right),
\end{align*}
where the step marked with (\raisebox{-0.5ex}{*}) follows by a proper choice of $\Delta$ minimizing the bound:
\begin{align*}
T - 120Ku^\frac{1}{\epsilon}\epsilon^{-1}\Delta^{-\frac{1+\epsilon}{\epsilon}}\log \frac{T}{2} = 0 \implies \Delta = \left( \frac{120 Ku^\frac{1}{\epsilon}\log \frac{T}{2}}{\epsilon T} \right)^{\frac{\epsilon}{1+\epsilon}},
\end{align*}
and step marked with (\raisebox{-0.5ex}{**}) follows by bounding simple numerical bounds.
\end{proof}

\clearpage
\section{Efficient Numerical Resolution of Equation \eqref{eq:M_hat}}\label{apx:comp}
In this appendix, we present a computationally efficient strategy that can be implemented in Algorithm \ref{alg:alg} to execute line \ref{line:M_hatAlg}, \ie the solution of the root-finding problem. In particular, to solve the equation:
\begin{align}
    f_{s}(\mathbf{X}';M,\delta) \coloneqq \frac{1}{s} \sum_{j\in[s]}\frac{\min \{(X'_j)^{2}, M^{2}\}}{M^2} - \frac{c \log \delta^{-1}}{s} = 0.\tag{\ref{eq:M_hat}}
\end{align}
\RestyleAlgo{ruled}
\LinesNumbered
\begin{algorithm2e}[h!]
\caption{Computationally Efficient Threshold Estimation.}\label{alg:numerical}
\DontPrintSemicolon
{\small
\SetKwInOut{Input}{input}
}
Reward set $\mathbf{X}' = \{X_1', \ldots, X_s'\}$, time counter $\tau$, machine tolerance $\eta>0$. \label{line:input}\\
Initialize counter $h\leftarrow 0$, initial guess $x_0 \leftarrow \eta$, initial value $y_0 \leftarrow f_s(\mathbf{X}'; x_0, \tau^{-3})$. \\
\While{$y_h > 0$}{ 
$x_{h+1} \leftarrow 2x_{h}$ \label{line:start2}\\
$y_{h+1} \leftarrow f_s(\mathbf{X}'; x_{h+1}, \tau^{-3})$ \\
$h \leftarrow h + 1$ \label{line:end2} \\
}
Return $x_{h}$ \label{line:return}.
\end{algorithm2e}

We propose Algorithm \ref{alg:numerical} to find an upper bound $\bar{M}_s(\tau^{-3})$ on the true solution $\widehat{M}_s(\tau^{-3})$ which is based on \emph{bisection}. The strategy works as follows. We provide the minimum numerical tolerance of our machine $\eta>0$, start from an initial guess $x_0 = \eta$, then, if this guess is an underestimation (\ie $f_s(\cdot,x_0)$ yields a positive value $y_0$) we proceed to iteratively double our guess until the real threshold has been passed (lines \ref{line:start2}-\ref{line:end2}). 
In line \ref{line:return}, we return the final guess $x_{h}$. If the initial guess is already an overestimation of the threshold (\ie $f_s(\cdot, x_0)$ yields a negative value $y_0$), we simply have $x_0 = x_h = \eta$.

We point out that, by construction, the output of Algorithm \ref{alg:numerical} can be \textit{at most} two times the true solution to Equation \eqref{eq:M_hat}, \ie $\bar{M}_s(\tau^{-3}) \le 2\widehat{M}_s(\tau^{-3})$. Thus, regret guarantees for Algorithm \ref{alg:alg} remain the same (up to numerical constants) even when performing this approximation of the threshold. In particular, in the proof of Theorem \ref{thr:upper_bound}, we can modify \eqref{eq:robust5} as follows:
\begin{align*}
    \bar{M}_{i}(t) \geq 2\left(\dfrac{u N_i(t-1)}{\log t^3}\right)^{\frac{1}{1+\epsilon}},
\end{align*}
and the final result remains the same up to multiplicative constants.

We now characterize the computational complexity of Algorithm \ref{alg:numerical}, \ie the maximum number of steps to be performed before returning a solution.
\begin{restatable}[Upper Bound on Algorithm \ref{alg:numerical} Number of Steps]{prop}{NumericalApprx}
\label{prop:numerical}
Let $\eta$ be the minimum numerical tolerance, and assume $\eta \le \widehat{M}_s(\tau^{-3})$. Then, in at most $\bar{h}_{\eta, \tau}(\epsilon,u)$ steps such that:
\begin{align*}
    \bar{h}_{\eta, \tau}(\epsilon,u) = \log_2\left(\frac{1}{\eta}\left(\frac{us}{\log(\tau^3)}\right)^\frac{1}{1+\epsilon}\right),
\end{align*}
Algorithm \ref{alg:numerical}, returns a solution $x_{\bar{h}_{\eta, \tau}(\epsilon,u)}$ s.t.
\begin{align*}
    \mathbb{P}\left( \frac{x_{\bar{h}_{\eta, \tau}(\epsilon,u)}}{\widehat{M}_s(\tau^{-3})} \in [1,2] \right) \ge 1-\frac{2}{\tau^3}.
\end{align*}
\end{restatable}

Proposition \ref{prop:numerical} states an upper bound for the number of steps of Algorithm \ref{alg:numerical} as a function of both $\epsilon$ and $u$. However, we remark that these two are not required as input to the numerical solver. Moreover, it emerges a dependence on the inverse of the numerical tolerance of the machine on which the algorithm is run. Thanks to the logarithm, this dependence hardly becomes an issue. If we consider a very small tolerance of $10^{-16}$ (which is the standard tolerance of many programming languages) the number of steps becomes:
\begin{align*}
    \bar{h}_{\eta, \tau}(\epsilon,u) = \log_2\left(\left(\frac{us}{\log(\tau^3)}\right)^\frac{1}{1+\epsilon}\right) + 16\log_2(10),
\end{align*}
which is totally reasonable.

\end{document}